\documentclass[twoside]{article}

\usepackage{ecj,palatino,epsfig,latexsym,natbib}

\usepackage{booktabs} 
\usepackage{todonotes}
\usepackage{amsmath}
\usepackage{newtxtext}
\usepackage{tikz}
\usepackage[linesnumbered,ruled,vlined]{algorithm2e}
\usepackage{multirow}
\usepackage{numprint}
\usepackage{balance}
\usepackage{amssymb}
\usepackage{amsthm}
\usepackage{url}
\usepackage{subfig}

\newtheorem{theorem}{Theorem}

\newtheorem{lemma}{Lemma}

\newcommand{\Bo}[0]{\mathbb{B}}

\npdecimalsign{.}

\begin{document}

\ecjHeader{x}{x}{xxx-xxx}{2021}{Dynastic Potential Crossover Operator}{F. Chicano, G. Ochoa, D. Whitley and R. Tin\'os}
\title{\bf Dynastic Potential Crossover Operator}  

\author{
\name{\bf Francisco Chicano} \hfill \addr{chicano@lcc.uma.es}\\
\addr{ITIS Software, University of Malaga, Spain}
\AND
\name{\bf Gabriela Ochoa} \hfill \addr{gabriela.ochoa@cs.stir.ac.uk}\\
\addr{University of Stirling, UK}
\AND
\name{\bf L. Darrell Whitley} \hfill \addr{whitley@cs.colostate.edu}\\
\addr{Colorado State University, USA}
\AND
\name{\bf Renato Tin\'os} \hfill \addr{rtinos@ffclrp.usp.br}\\
\addr{University of Sao Paulo, Brazil}
}

\maketitle

\begin{abstract}
An optimal recombination operator for two parent solutions provides the best solution among those that take the value for each variable from one of the parents (gene transmission property). If the solutions are bit strings, the offspring of an optimal recombination operator is optimal in the smallest hyperplane containing the two parent solutions.
Exploring this hyperplane is computationally costly, in general, requiring exponential time in the worst case. However, when the variable interaction graph of the objective function is sparse, exploration can be done in polynomial time. 

In this paper, we present a recombination operator, called Dynastic Potential Crossover (DPX), that runs in polynomial time and behaves like an optimal recombination operator for low-epistasis combinatorial problems. We compare this operator, both theoretically and experimentally, with traditional crossover operators, like uniform crossover and network crossover, and with two recently defined efficient recombination operators: partition crossover and articulation points partition crossover. The empirical comparison uses NKQ Landscapes and MAX-SAT instances. DPX outperforms the other crossover operators in terms of quality of the offspring and provides better results included in a trajectory and a population-based metaheuristic, but it requires more time and memory to compute the offspring.
\end{abstract}


\section{Introduction}
\label{sec:intro}
\emph{Gene transmission}~\citep{Radcliffe1994} is a popular property commonly fulfilled by many recombination operators for genetic algorithms.
When the solutions are represented by a set of variables taking values from finite alphabets (possibly different from each other) with no constraints among the variables, this property implies that any variable in any offspring will take the value of the same variable in one of the parents.
In particular, the variables having the same value for both parents will have the same value in all the offspring (i.e., the \emph{respect} property~\citep{Radcliffe1994} is obeyed). The other (differing) variables will take one of the values coming from a parent solution. The gene transmission property is a formalization of the idea that taking (good) features from the parents should produce better offspring. This probably explains why most of the recombination operators try to fulfill this property or some variant of it. 
The set of all the solutions that can be generated by a recombination operator from two parents is called \emph{dynastic potential}. 
If we denote with $h(x,y)$ the Hamming distance (number of differing variables) between two solutions $x$ and $y$, the cardinality of the largest dynastic potential of a recombination operator fulfilling the gene transmission property is $2^{h(x,y)}$. The dynastic potential of uniform crossover has this size. The dynastic potential of single-point crossover has size $2h(x,y)$ because it takes from 1 to $h(x,y)$ consecutive differing bits from any of the two parents. In two-point crossover, two cut points are chosen. If the first one is before the first differing bit or the second one is after the last differing bit, the operator behaves like the single-point crossover. Otherwise, the two cut points are chosen among the $h(x,y)-1$ positions between differing bits and for each cut the bits can be taken from any of the two parents, providing $2\binom{h(x,y)-1}{2}$ additional solutions to the single-point crossover dynastic potential. Thus, the size of the two-point crossover dynastic potential is $2h(x,y)+2\binom{h(x,y)-1}{2}=h(x,y)^2 - h(x,y) + 2$. In general, $z$-point crossover has a dynastic potential of size $\Theta(h(x,y)^z)$ when $z$ is small compared to the number of decision variables, $n$. 

An \emph{optimal recombination operator}~\citep{Eremeev:Kovalenko2013} obtains a best offspring from the largest dynastic potential of a recombination operator fulfilling the gene transmission property, which has size $2^{h(x,y)}$. In the worst case, such a recombination operator is computationally expensive, because finding a best offspring in the largest dynastic potential of an NP-hard problem is an NP-hard problem. A proof of the NP-hardness is given by \citet{Eremeev:Kovalenko2013}, but it can also be easily concluded from the fact that applying an optimal recombination operator to two complementary solutions (e.g., all zeroes and all ones) is equivalent to solving the original problem, and the NP-hardness is derived from the NP-hardness of the original problem. 

We propose a recombination operator, named \emph{Dynastic Potential Crossover} (DPX),
that finds a best offspring of the largest dynastic potential if the objective function $f$ has low epistasis, that is, if the number of nonlinear interactions among variables is small, typically $\Theta(n)$. 
In particular, we assume that the objective function $f$ is defined on $n$ binary variables and has $k$-bounded epistasis. This means that $f$ can be written as a sum of $m$ subfunctions $f_{\ell}$, each one depending on at most $k$ variables:
\begin{equation}
\label{eqn:fn}
f(x) = \sum_{\ell=1}^{m} f_{\ell}(x_{i(\ell,1)},x_{i(\ell,2)},\ldots,x_{i(\ell,k_{\ell})}) ,
\end{equation}
where $i(\ell,j)$ is the index of the $j$-th variable in subfunction $f_{\ell}$ and $k_{\ell} \leq k$. These functions have been named Mk Landscapes by \citet{Whitley2016ecj}. 

DPX has a non-negative integer parameter, $\beta$, bounding the exploration of the dynastic potential. The precise condition to certify that the complete dynastic potential has been explored depends on $\beta$ and will be presented in Section~\ref{subsec:complexity}. The worst-time complexity of DPX is $O(4^{\beta}(n+m)+n^2)$ for Mk Landscapes. 

DPX uses the variable interaction graph of the objective function to simplify the evaluation of the $2^{h(x,y)}$ solutions in the dynastic potential by using dynamic programming. The ideas for this efficient evaluation date back to the basic algorithm by \citet{Hammer1963} for variable elimination and are also commonly used in operations over Bayesian networks~\citep{Bodlaender2005}. DPX requires more than just the fitness values of the parent solutions to do the job and, thus, it is framed in the so-called gray box optimization~\citep{Whitley2016ecj}.

Recently defined crossover operators similar to DPX are \emph{partition crossover} (PX)~\citep{TinosWhitley2015} and \emph{articulation points partition crossover} (APX)~\citep{Chicano2018gecco}. Although they were proposed to work with pseudo-Boolean functions, they can also be applied to the more general representation of variables defined over a finite alphabet. We will follow the same approach here. In the rest of the paper we will focus on a binary representation, where each variable takes value in the set $\Bo=\{0,1\}$. But most of the results, claims and the operator itself can be applied to a more general solution representation, where variables take values from finite alphabets.
PX and APX also use the variable interaction graph of the objective function to efficiently compute a good offspring among a large number of them. 
PX and APX produced excellent performance in different problems~\citep{ChicanoWOT17,Chicano2018gecco,Chen2018gecco,Tinos2018tevc}. When combined with other gray box optimization operators, PX and APX were capable of optimizing instances with one million variables in seconds.
We compare DPX with these two operators from a theoretical and experimental point of view.  The new recombination operator is also compared to uniform crossover and network crossover~\citep{Hauschild2010}, which also uses the variable interaction graph.

This work extends a conference paper~\citep{Chicano2019EvoCOP} where DPX was first presented. The new contributions of this journal version are as follows:
\begin{itemize}
\item We added a theorem to prove the correctness of the dynamic programming algorithm.
\item We revised the text to clarify some terms, add extended explanations of some details of the operator and provide more examples.
\item We provided new experiments to compare the performance of DPX with the one of other crossover operators when they are applied to random solutions. The other crossover operators are uniform crossover, network crossover, partition crossover and articulations point partition crossover.
\item We included DPX in an evolutionary algorithm, to check its behaviour in this algorithm.
\item We compared the five crossover operators inside a trajectory-based and a population-based metaheuristic used to solve two NP-hard pseudo-Boolean optimization problems: NKQ Landscapes and MAX-SAT.
\item We modified the code of DPX to improve its performance and updated the source code repository\footnote{Available at \url{https://github.com/jfrchicanog/EfficientHillClimbers}} with the new code of the evolutionary algorithm, the newly implemented crossover operators used in the comparison and the algorithms to process the results.
\item We added a local optima network analysis \citep{OchoaV-joh18} of DPX included within an iterated local search framework to better understand its working principles.
\end{itemize}

The paper is organized as follows. Section~\ref{sec:background} presents the required background to understand how DPX works. The proposed recombination operator is presented in Section~\ref{sec:dpx}. Section~\ref{sec:experiments} describes the experiments and presents the results and, finally, Section~\ref{sec:conclusions} concludes the paper.

\section{Background}
\label{sec:background}

In our \emph{gray box optimization} setting, the optimizer can independently evaluate the set of $m$ subfunctions in Equation~\eqref{eqn:fn}, i.e., it is able to evaluate $f_\ell$ when the values of $x_{i(\ell,1)} \ldots x_{i(\ell,k_{\ell})}$ are given; and it also knows the variables each $f_\ell$ depends on. This contrasts with \emph{black box optimization}, where the optimizer can only evaluate full solutions $(x_0,...,x_{n-1})$ and get their fitness value $f(x)$. We assume that we do not know the internal details of $f_\ell$ (we can only evaluate it), and this is why we call it gray box and not white box. 

\subsection{Variable interaction graph}
\label{subsec:vig}

The \emph{variable interaction graph} (VIG)~\citep{Whitley2016ecj} is a useful tool that can be constructed under gray box optimization. It is a graph $VIG = (V,E)$, where $V$ is the set of variables and $E$ is the set of edges representing all pairs of variables $(x_i , x_j)$ having \emph{nonlinear interactions}. 
The set $\Bo^n$ contains all the binary strings of length $n$. Given a set of variables $Y$, the notation $1_Y^n$ is used to represent a binary string of length $n$ with 1 in the positions of the variables in $Y$ and zero in the rest, for example, $1^3_{\{0,2\}}=101$. If the set $Y$ has only one element, we will omit the curly brackets in the subscript, e.g., $1^4_3 = 1^4_{\{3\}}=0001$. Observe that the first index of the binary strings is 0.
When the length is clear from the context we omit $n$. The operator $\oplus$ is a bitwise exclusive OR. 

We say that variables $x_i$ and $x_j$ have a nonlinear interaction when the expression $\Delta_{i}f(x) = f(x \oplus 1_{i})-f(x)$ does depend on $x_j$.
Checking the dependency of $\Delta_{i}f(x)$ on $x_j$ can be computationally expensive because it requires to evaluate the expression on all the strings in $\Bo^{n-1}$. 
There are other approaches to find the nonlinear interactions among variables. First, assuming that every pair of variables appearing together in a subfunction has a nonlinear interaction. 
It is not necessarily true that there is a nonlinear interaction among variables appearing as arguments in the same subfunction, but adding extra edges to $E$ does not break the correctness of the operators based on the VIG and requires only a very simple check that is computationally cheap. The graph obtained this way is usually called \emph{co-ocurrence graph}.
A second, and precise, approach to determine the nonlinear interactions is to apply the Fourier transform~\citep{Terras1999}, and then look at every pair of variables to determine if there is a nonzero Fourier coefficient associated to a term with the two variables. This second method is precise and can be done in $\Theta(m 4^{k})$ time if we know the variables appearing in each subfunction $f_\ell$ and we can evaluate each $f_\ell$ independently (our gray box setting). The interested reader can see the work of~\citet{Whitley2016ecj} and~\citet{Rana1998} for a more detailed description of this second approach.

The first approach is specially useful when $k$ is relatively large, because it requires time $\Theta(m k^2)$, which is polynomial in $k$, in contrast to the exponential time in $k$ of the second approach. In some problems, like MAX-SAT, we know that the co-ocurrence graph (obtained by the first approach) is exactly the variable interaction graph (obtained by the second approach). In some other problems, like NK Landscapes, both are the same with high probability. In these two cases, it makes sense to use the first approach and work with the co-ocurrence graph even in the case that $k$ is small enough to compute the Fourier transform in a reasonable time. In the rest of the paper when we use the term variable interaction graph we could replace it by the co-ocurrence graph.

An example of the construction of the variable interaction graph for a function with $n = 18$ variables (numbered from 0 to 17) and $k=3$ is given below. We will refer to the variables using numbers,  e.g., $9 = x_9$. The objective function is the sum over the following 18 subfunctions:


\begin{center}
\begin{tabular} {llll}
 $f_0(0, 6, 14)$ &   $f_5(5, 4, 2)$    & $f_{10}(10, 2, 17) $   & $f_{15}(15, 7, 13) $   \\
 $f_1(1, 0, 6)$  &   $f_6(6, 10, 13)$  & $f_{11}(11, 16, 17) $  & $f_{16}(16, 9, 11) $   \\
 $f_2(2, 1, 6)$  &   $f_7(7, 12, 15)$  & $f_{12}(12, 10, 17) $  & $f_{17}(17, 5, 2) $   \\
 $f_3(3, 7, 13)$ &   $f_8(8,  3, 6)$   & $f_{13}(13, 12, 15) $  \\
 $f_4(4, 1, 14)$ &   $f_9(9, 11, 14)$  & $f_{14}(14, 4, 16) $  \\
\end{tabular}
\end{center}

From these subfunctions, assume we extract the nonlinear interactions that are shown in Figure~\ref{fig:vig}. In this example, every pair of variables that appear together in a subfunction has a nonlinear interaction.  

\begin{figure}[!ht]
\centering
\tikzstyle{estilo}=[circle,draw=black]
%
\begin{tikzpicture}[scale=0.50, every node/.style={scale=0.8}]
\node[estilo] (x0) at (6,1) {$0$};
\node[estilo] (x1) at (6,2.5) {$1$};
\node[estilo] (x2) at (6,4) {$2$};
\node[estilo] (x3) at (14,2) {$3$};
\node[estilo] (x4) at (4,2) {$4$};
\node[estilo] (x5) at (4,4) {$5$};
\node[estilo] (x6) at (8,1) {$6$};
\node[estilo] (x7) at (14,7) {$7$};
\node[estilo] (x8) at (10,0.5) {$8$};
\node[estilo] (x9) at (2,1.5) {$9$};
\node[estilo] (x10) at (9,4) {${10}$};
\node[estilo] (x11) at (0,2) {${11}$};
\node[estilo] (x12) at (10,6) {${12}$};
\node[estilo] (x13) at (12,4) {${13}$};
\node[estilo] (x14) at (4,-0.5) {${14}$};
\node[estilo] (x15) at (12,5.5) {${15}$};
\node[estilo] (x16) at (2,3) {${16}$};
\node[estilo] (x17) at (4,6) {${17}$};
\draw (x0) -- (x1);
\draw (x0) -- (x6);
\draw (x0) -- (x14);
\draw (x1) -- (x2);
\draw (x1) -- (x4);
\draw (x1) -- (x6);
\draw (x1) -- (x14);
\draw (x2) -- (x4);
\draw (x2) -- (x5);
\draw (x2) -- (x6);
\draw (x2) -- (x10);
\draw (x2) -- (x17);
\draw (x3) -- (x6);
\draw (x3) -- (x7);
\draw (x3) -- (x8);
\draw (x3) -- (x13);
\draw (x4) -- (x5);
\draw (x4) -- (x14);
\draw (x4) -- (x16);
\draw (x5) -- (x17);
\draw (x6) -- (x8);
\draw (x6) -- (x10);
\draw (x6) -- (x13);
\draw (x6) -- (x14);
\draw (x7) -- (x12);
\draw (x7) -- (x13);
\draw (x7) -- (x15);
\draw (x9) -- (x11);
\draw (x9) -- (x14);
\draw (x9) -- (x16);
\draw (x10) -- (x12);
\draw (x10) -- (x13);
\draw (x10) -- (x17);
\draw (x11) -- (x17);
\draw (x11) -- (x16);
\draw (x11) -- (x14);
\draw (x12) -- (x13);
\draw (x12) -- (x15);
\draw (x12) -- (x17);
\draw (x13) -- (x15);
\draw (x14) -- (x16);
\draw (x16) -- (x17);
\end{tikzpicture}
\caption{Sample variable interaction graph (VIG).} 
\label{fig:vig}
\end{figure}
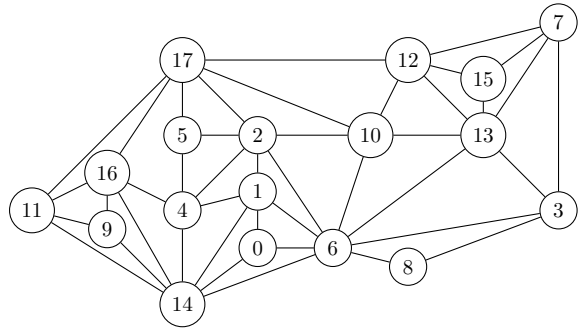

\subsection{Recombination Graph}
\label{subsec:recombination}

Let us assume that we have two solutions to recombine. We call these two solutions  {\emph{red}} and {\emph{blue}} parents. All the variables with the same value in both parents will also share the same value in the offspring and the solutions in the dynastic potential will be in a hyperplane determined by the common variables. For example, let the two parents be 
\begin{align*}
\text{red} &= 000000000000000000, \\
\text{blue} &= 111101011101110110,
\end{align*}
in our sample function presented in Section~\ref{subsec:vig}. Therefore,  $x_4$, $x_6$, $x_{10}$, $x_{14}$, and $x_{17}$ are identical in both parents. The rest of the variables are different. Both parents reside in a hyperplane denoted by $H={*}{*}{*}{*}0{*}0{*}{*}{*}0{*}{*}{*}0{*}{*}0$ where $*$ denotes the variables that are different in the two solutions, and $0$ marks the positions where they have the same variable values.

We use the hyperplane $H={*}{*}{*}{*}0{*}0{*}{*}{*}0{*}{*}{*}0{*}{*}0$ to decompose the VIG in order to produce a {\em recombination graph}.  We remove all the variables (vertices) that have the same ``shared variable assignments'' and also remove all edges that are incident on the vertices corresponding to these variables. This produces the  recombination graph shown in Figure~\ref{fig:recom2}.

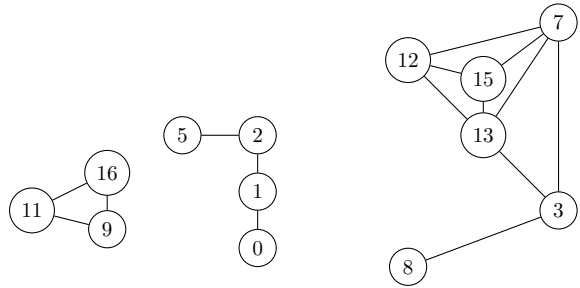
\begin{figure}[!ht]
    \centering
\tikzstyle{estilo}=[circle,draw=black]
%
\begin{tikzpicture}[scale=0.50, every node/.style={scale=0.8}]
\node[estilo] (x0) at (6,1) {$0$};
\node[estilo] (x1) at (6,2.5) {$1$};
\node[estilo] (x2) at (6,4) {$2$};
\node[estilo] (x3) at (14,2) {$3$};
\node[estilo] (x5) at (4,4) {$5$};
\node[estilo] (x7) at (14,7) {$7$};
\node[estilo] (x8) at (10,0.5) {$8$};
\node[estilo] (x9) at (2,1.5) {$9$};
\node[estilo] (x11) at (0,2) {${11}$};
\node[estilo] (x12) at (10,6) {${12}$};
\node[estilo] (x13) at (12,4) {${13}$};
\node[estilo] (x15) at (12,5.5) {${15}$};
\node[estilo] (x16) at (2,3) {${16}$};
\draw (x0) -- (x1);
\draw (x1) -- (x2);
\draw (x2) -- (x5);
\draw (x3) -- (x7);
\draw (x3) -- (x8);
\draw (x3) -- (x13);
\draw (x7) -- (x12);
\draw (x7) -- (x13);
\draw (x7) -- (x15);
\draw (x9) -- (x11);
\draw (x9) -- (x16);
\draw (x11) -- (x16);
\draw (x12) -- (x13);
\draw (x12) -- (x15);
\draw (x13) -- (x15);
\end{tikzpicture}
\caption{Recombination graph for the solutions (parents)
$\text{red} = 000000000000000000$  and $\text{blue} =  111101011101110110$.}
\label{fig:recom2}
\end{figure}

The recombination graph also defines a reduced evaluation function.  This new evaluation function is linearly separable, 
and decomposes into $q \leq m$ subfunctions defined over the connected components of the recombination graph. In our example:
\[ g(x') = a + g_1(9, 11, 16) + g_2(0, 1, 2, 5) + g_3(3, 7, 8, 12, 13, 15)  ,\]
where $g(x') = f|_H(x')$ and $x'$ are solutions restricted to a subspace of the hyperplane $H$ that contains 
the parent strings as well as the full dynastic potential.  
The constant $a = f(x') - \sum_{i=1}^{3} g_i(x')$ depends on the common variables. 
Every recombination graph with $q$ connected components induces a new {\em separable} function $g(x')$ that is defined as:
\begin{equation}
\label{eqn:px-eqn}
g(x') = a + \sum_{i=1}^{q} g_i(x') .
\end{equation}

Partition crossover (PX), defined by~\cite{TinosWhitley2015}, generates an offspring, when recombining two parents, based on the recombination graph: all of the variables in the same recombining component in the recombination graph are inherited together from one of the two parents.
Partition crossover selects the decision variables from one or another parent yielding the best partial evaluation for each subfunction $g_i(x')$. This way, PX obtains a best offspring among $2^q$ in $O(q)$ time, which is a remarkable result. The efficiency of PX depends on the number of connected components $q$ in the recombination graph. Larger values for $q$ provide a better performance. We wonder if this number can be large in pseudo-Boolean problems with interest in practice. We provide a positive answer with the VIG and recombination graph in Figure~\ref{fig:maxsat-recomb}. It shows a sample recombination graph with \numprint{1087} connected components of a real SAT instance of the 2014 Competition.\footnote{ \url{http://www.satcompetition.org/2014/}} The graph was generated by~\citet{ChenWhitley2017}.

\begin{figure}
\centering
\includegraphics[scale=1]{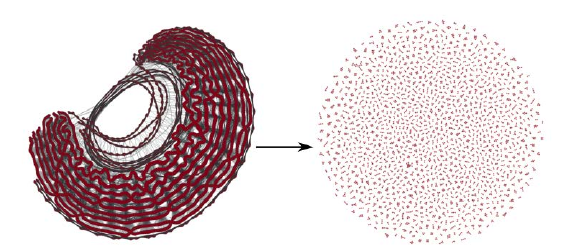}
\caption{Variable interaction graph (left) and recombination graph (right), generated by \citet{ChenWhitley2017}, for instance atco\_enc3\_opt1\_13\_48 (\numprint{1067657} variables) from the SAT Competition 2014. The recombination graph contains 1087 connected components.}
\label{fig:maxsat-recomb}
\end{figure}

Articulation points partition crossover (APX)~\citep{Chicano2018gecco} goes further and finds the \emph{articulation points} of the recombination graphs. Articulation points are variables whose removal increases the number of connected components. 
A \emph{bi-connected} component of a graph is a maximal subgraph with the property that  there are two vertex-disjoint paths between any two vertices. Articulation points join several bi-connected components.
Variables $x_1$, $x_2$ and $x_3$ are articulation points in our example (see Figure~\ref{fig:recom2}) and the subgraphs induced by the vertex sets $\{x_5, x_2\}$ and $\{x_3, x_7, x_{12}, x_{13}, x_{15}\}$ are examples of bi-connected components. Then, APX efficiently simulates what happens when the articulation points are removed, one at a time, from the recombination graph by flipping the variable in each of the parent solutions before applying PX, and the best solution is returned as offspring. In our example, APX would work as follows. First, it applies PX to the red solution and a copy of the blue solution where the variable $x_1$ is flipped and stores the best children. Then, it applies PX to the blue solution and a copy of the red solution with variable $x_1$ flipped. The same process is repeated with flips in variables $x_2$ and $x_3$ (the other articulation points). Finally, it applies PX to the original red and blue solutions. APX returns the best solution of all the applications of PX. The key ingredient of APX is that all these computations do not require repeated applications of PX. With the appropriate data structures, all the computations can be done in $O(n^2+m)$, the same complexity of PX, for any choice of parents in Mk Landscapes.

\section{Dynastic potential exploration}
\label{sec:dpx}

The proposed dynastic potential crossover operator (DPX) takes the ideas of PX and APX even further. DPX starts from the recombination graph, like the one in Figure~\ref{fig:recom2}. Then, DPX tries to exhaustively explore all the possible combinations of the parent values in the variables of each connected component to find the optimal recombination regarding the hyperplane $H$.
This exploration is not done by brute force, but using dynamic programming. Following with our example, in order to compute the best combination for the variables $x_9$, $x_{11}$ and $x_{16}$, we need to enumerate the 8 ways of taking each variable from each parent, and this is not better than brute force. However, the component containing variables $x_0$, $x_1$, $x_2$ and $x_5$ forms a path. In this case, we can store in a table what is the best option for variable $x_0$ when any of the two possible values for variable $x_1$ are selected. Then, we can store in the same table what is the value of the sum of subfunctions depending only on $x_0$ and $x_1$ (and possibly common variables eliminated in the recombination graph). After this step, we can consider that variable $x_0$ has been removed from the problem, and we can proceed in the same way with the rest of the variables in order: $x_1$, $x_2$ and $x_5$. Finally,  12 evaluations  are necessary, instead of the 16 required by brute force. In general, for a path of length $\ell$, dynamic programming requires $4(\ell-1)$ evaluations while brute force requires $2^\ell$ evaluations.

The idea of variable elimination using dynamic programming dates back to the 1960's and  the basic algorithm by~\cite{Hammer1963}. 
The problem of variable elimination has also been studied in other contexts, like Gaussian elimination~\citep{Tarjan:Yannakis1984} and Bayesian networks~\citep{Bodlaender2005}. In fact, we utilize the ideas for computing the \emph{junction tree} in Bayesian networks. First, a \emph{chordal graph} is obtained from the recombination graph using the \emph{maximum cardinality search}  and a \emph{fill-in procedure} to add the missing edges. 
Then the \emph{clique tree} (or junction tree) is computed, which will fix the order in which the variables are eliminated using dynamic programming.
After assigning the subfunctions to the cliques in the clique tree, dynamic programming is applied to find a best offspring, which is later reconstructed using the information computed in tables during dynamic programming. The runtime of the variable elimination method depends, among others, on the number of missing edges added by the fill-in procedure. Unfortunately, finding the minimum fill-in is an  NP-hard problem~\citep{Bodlaender2005}. Thus, we do not try to eliminate the variables in the most efficient way, but we apply algorithms that are efficient in finding a variable elimination order. 
Our proposal, DPX, is the first, to the best of our knowledge, applying these well-known ideas to design a recombination operator. There is also a difference between our approach and the variable elimination methods in the literature: we introduce a parameter $\beta$ to limit the exploration of the variables (see Section~\ref{subsec:complexity}). The high-level pseudocode of DPX is outlined in Algorithm~\ref{alg:dpx}. In the next subsections we will detail each of these steps.

\begin{algorithm}[!ht]
\caption{Pseudocode of DPX}
\label{alg:dpx}
\KwData{two parents $x$ and $y$}
\KwResult{one offspring $z$}
Compute the recombination graph of $x$ and $y$ as in PX \citep{TinosWhitley2015}\;
\label{lin:recombination-graph}
Apply maximum cardinality search~\citep{Tarjan:Yannakis1984}\;
\label{lin:mcs}
Apply the fill-in procedure to make the graph chordal~\citep{Tarjan:Yannakis1984}\;
\label{lin:fillin}
Apply the clique tree construction procedure~\citep{Galinier1995}\;
\label{lin:clique-tree}
Assign subfunctions to cliques in the clique tree\;
\label{lin:subfunctions-assignment}
Apply dynamic programming to find the offspring (see Algorithm~\ref{alg:dynp})\;
\label{lin:dynp}
Build $z$ using the tables filled by dynamic programming\;
\label{lin:reconstruction}
\end{algorithm}

\subsection{Chordal graphs}
\label{subsec:chordal}

In Algorithm~\ref{alg:dpx}, after finding the recombination graph (Line~\ref{lin:recombination-graph}), each connected component is transformed into a chordal graph (Lines~\ref{lin:mcs} and~\ref{lin:fillin}), if it is not already one. 
A \emph{chordal graph} is a graph where all the cycles of length four or more have a chord (edge joining two nodes not adjacent in the cycle). All the connected components in Figure~\ref{fig:recom2} are chordal graphs. \citet{Tarjan:Yannakis1984} provided algorithms to test if a graph is chordal and add new edges to make it chordal if it is not. Their algorithms run in time $O(n+e)$, where $n$ is the number of nodes in the graph and $e$ is the number of edges. In the worst case the complexity is $O(n^2)$. The first step to check the chordality is to number the nodes using \emph{maximum cardinality search} (MCS). This algorithm numbers each node in descending order, choosing in each step the unnumbered node with more numbered neighbors and solving the ties arbitrarily. The number associated to node $u$ is denoted with $\gamma(u)$. Figure~\ref{fig:mcs} (left) shows the result of applying MCS to the third connected component of Figure~\ref{fig:recom2}, where we started numbering node $12$.

\begin{figure}[!ht]
\centering
\tikzstyle{estilo}=[circle,draw=black]
\tikzstyle{numbers}=[text=red]
\tikzstyle{clique}=[rectangle,draw=black,align=left]
%
\begin{tikzpicture}[scale=0.50, every node/.style={scale=0.8}]
\node[estilo] (x3) at (14,2) {$3$};
\node[estilo] (x7) at (14,7) {$7$};
\node[estilo] (x8) at (10,0.5) {$8$};
\node[estilo] (x12) at (10,6) {${12}$};
\node[estilo] (x13) at (12,4) {${13}$};
\node[estilo] (x15) at (12,5.5) {${15}$};
\node[numbers] (l3) at (14.5,2.5) {$2$};
\node[numbers] (l7) at (14.5,6.5) {$5$};
\node[numbers] (l8) at (9.4,1) {$1$};
\node[numbers] (l12) at (9.3,6.5) {$6$};
\node[numbers] (l13) at (11.3,3.5) {$3$};
\node[numbers] (l15) at (11.3,6) {$4$};
\draw (x3) -- (x7);
\draw (x3) -- (x8);
\draw (x3) -- (x13);
\draw (x7) -- (x12);
\draw (x7) -- (x13);
\draw (x7) -- (x15);
\draw (x12) -- (x13);
\draw (x12) -- (x15);
\draw (x13) -- (x15);
\end{tikzpicture} \hspace{1cm}
\begin{tikzpicture}[scale=0.50, every node/.style={scale=0.8}]
\node[clique] (c1) at (0,10) {$C_1= \{7,12,13,15\}$\\$S_1=\emptyset$\\$R_1=\{7,12,13,15\}$};
\node[clique] (c2) at (0,7) {$C_2= \{3,7,13\}$\\$S_2=\{7,13\}$\\$R_2=\{3\}$};
\node[clique] (c3) at (0,4) {$C_3= \{3,8\}$\\$S_3=\{3\}$\\$R_3=\{8\}$};
\draw [->] (c2) -- (c1);
\draw [->] (c3) -- (c2);
\end{tikzpicture}
\caption{Maximum cardinality search applied to the third connected component of Figure~\ref{fig:recom2} (left) and clique tree with the sets $S_i$ and $R_i$ (right).}
\label{fig:mcs}
\end{figure}
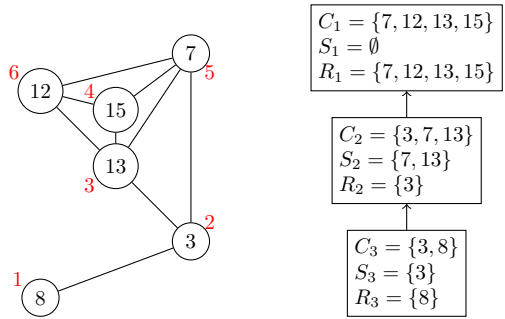

If the graph is chordal then MCS will provide a numbering of the nodes $\gamma$ such that for each triple of nodes $u$, $v$ and $w$, with $(u,v), (u,w) \in E$ and $\gamma(u) < \min\{\gamma(v),\gamma(w)\}$, it happens that $(v,w) \in E$. If this is not the case, the graph is not chordal. A \emph{fill-in} algorithm tests this condition and adds the required edges to make the graph chordal. This algorithm runs in $O(n+e')$ time, where $e'$ is the number of edges in the final chordal graph. Again, in the worst case, the complexity is $O(n^2)$. These two steps, MCS and fill-in, can be applied to each connected component separately or to the complete recombination graph with the same result~\citep{Tarjan:Yannakis1984}.

\subsection{Clique Tree}
\label{subsec:clique-tree}

Dynamic programming is used to exhaustively explore all the variables in each clique\footnote{We will use the term \emph{clique} to refer to a maximal complete subgraph, as the cited literature does. However, the term clique is sometimes used to refer to any complete subgraph (not necessarily maximal).} in the chordal graph. The maximum size of a clique in the chordal graph is an upper bound of its treewidth, and determines the complexity of applying dynamic programming to find the optimal solution. A clique tree of a chordal graph is a tree where the nodes are cliques and for any variable appearing in two of such cliques, the path among the two cliques in the tree is composed of cliques containing the variable (junction tree property). We can also identify a clique tree with a tree-decomposition of the chordal graph~\citep{Bodlaender2005}. This clique tree will determine the order in which the variables can be eliminated. 

Starting from the chordal graph provided in the previous steps, we apply an algorithm by \citet{Galinier1995} to find the clique tree $T$ (Line~\ref{lin:clique-tree} in Algorithm~\ref{alg:dpx}). This algorithm runs in $O(n+e')$ time and finds all the $O(n)$ cliques of the chordal graph. In a chordal graph, the number of cliques cannot exceed the number of nodes $n$ of the graph~\citep{Galinier1995}. 
The cliques will be identified with the sets of variables it contains, $C_i$, where $i$ is an integer index for the clique that increases when a clique is discovered by the algorithm. An edge joining two cliques in the clique tree can be labeled with a \emph{separator}, which is the intersection of the variables in both cliques. A clique $C_i$ is parent of a clique $C_j$ if they are joined by an edge and $i<j$. The set of child cliques of a clique $i$ is denoted with $ch(i)$. Although separators are associated to edges, in each clique $C_i$ we highlight a particular separator, the separator with its parent clique, and we will use $S_i$ to denote it. If a clique $C_i$ has no parent, then $S_i=\emptyset$.
The \emph{residue}, $R_i$, in a clique $C_i$ is the set of variables of $C_i$ that are not in the separator with its parent, $S_i$. In each clique $C_i$, the residue, $R_i$, and the separator with the parent, $S_i$, forms a partition of the variables in $C_i$. Due to the junction tree property, for each variable $x_i$, the cliques that contain it forms a connected component in the clique tree $T$. The variable is in the set $S_j$ of all the cliques $j$ in the connected component except in the ancestor of all of them (with the lowest index $j$), where $x_i$ is member of its residue $R_j$. Thus, each variable is only in the residue of one clique. 
In Figure~\ref{fig:mcs} (right) the residues, $R_i$, and separators with the parent, $S_i$, for all the cliques of the third connected component of Figure~\ref{fig:recom2} are shown.

After computing the clique tree, all the subfunctions $f_{\ell}$ depending on a nonempty set $V_{\ell}^d$ of differing variables must be assigned to one (and only one) clique $i$ where $V_{\ell}^d \subseteq C_i$ (Line~\ref{lin:subfunctions-assignment} in Algorithm~\ref{alg:dpx}). They will be evaluated when this clique is processed. There can be more than one clique where the subfunction can be assigned. 
All of them are valid for a correct evaluation, but the clique with less variables is preferred to reduce the runtime. 
We denote with $F_{i}$ the set of subfunctions assigned to clique $C_i$. 

An optimal offspring is found in Algorithm~\ref{alg:dynp} by exhaustively exploring all the variable combinations in each clique $C_i$ and storing the best ones. Before describing the algorithm we need to introduce some additional notation. 
The operator $\wedge$ is a bitwise AND and the expression $\Bo^n \wedge 1_Y$ denotes the set of binary strings of length $n$ with zero in all the variables not in $Y$. 

For each combination of the variables in the separator $S_i$ (Line~\ref{lin:si-iteration} in Algorithm~\ref{alg:dynp}), all the combinations of the variables in the residue $R_i$ are considered (Line~\ref{lin:ri-iteration} in Algorithm~\ref{alg:dynp}) and evaluated over the subfunctions assigned to the clique (Lines~\ref{lin:subfn-evaluation-start}-\ref{lin:subfn-evaluation-end}) and their child cliques (Lines~\ref{lin:children-evaluation-start}-\ref{lin:children-evaluation-end}). Then, the best combination for the residue $R_i$ is stored in the \texttt{variable}[$i$] array\footnote{What is really stored in Algorithm~\ref{alg:dynp} is the change $v$ of the variables in $R_i$ over the parent solution $x$, but this can be considered an implementation detail.} (Line~\ref{lin:variable-assignment}) and its value in the \texttt{value}[$i$] array (Line~\ref{lin:value-assignment}). The evaluation in post-order makes it possible to have the \texttt{value}[$j$] array of the child cliques filled when they are evaluated in Line~\ref{lin:child-value-eval}.
At this point, variables in the residue $R_i$ can be obviated (eliminated) in the rest of the computation.
When the separator $S_i=\emptyset$, $w$ only takes one value, the string with $n$ zeroes. This happens in the root of the clique tree, and its effect is to iterate over all the variables combinations for $R_i=C_i$ to find the best value.
The \texttt{variable} array will be used in the reconstruction of the offspring solution (Line~\ref{lin:reconstruction} in Algorithm~\ref{alg:dpx}). 

\begin{algorithm}[!ht]
\caption{Optimal offspring computation}
\label{alg:dynp}
\KwData{Clique tree $T$ with all cliques $C_i$, parent solution $x$}
\KwResult{Arrays variable and value}
\For{$C_i \in T$ in post-order}{
	\For{$w \in \Bo^n \wedge 1_{S_i}$}{
		\label{lin:si-iteration}
    	value[$i$][$w$] = $-\infty$\;
    	\For{$v \in \Bo^n \wedge 1_{R_i}$}{
    		\label{lin:ri-iteration}
        	aux = 0\;
        	\For{$f \in F_{i}$}{
        	\label{lin:subfn-evaluation-start}
            	aux = aux + $f(x \oplus w \oplus v)$\;
            }
            \label{lin:subfn-evaluation-end}
            \For{$j \in ch(i)$}{
            \label{lin:children-evaluation-start}
            	aux = aux + value[$j$][$(w \oplus v) \wedge 1_{S_j}$]\;
            	\label{lin:child-value-eval}
            }
            \label{lin:children-evaluation-end}
            \If{{\rm aux} $>$ {\rm value[}$i${\rm ][}$w${\rm ]}}{
            	variable[$i$][$w$] = $v$\;
                \label{lin:variable-assignment}
            	value[$i$][$w$] = aux\;
            	\label{lin:value-assignment}
            }
        }
    }
    \label{lin:si-iteration-end}
}
\end{algorithm}

Following with our previous example in Figure~\ref{fig:mcs}, first, clique $C_3$ is evaluated. For all the values of $x_3$ (variable in $S_3$), and all the values of $x_8$ (variable in $R_3$) the subfunctions in $F_{3}$ are evaluated. All these subfunctions depend only on $x_3$, $x_8$ and variables with common values in both parents. For each value $x_3=0,1$, the array \texttt{value}[3] is filled with the maximum value of the sum of the subfunctions in $F_{3}$ for the two values of $x_8=0,1$. The array \texttt{variable}[3] will store the value of $x_8$ for which the maximum is obtained in each case.
After $C_3$ has been evaluated, variable $x_8$ is eliminated. Now clique $C_2$ is evaluated, and \texttt{value}[2] and \texttt{variable}[2] are filled for each combination of $x_7$ and $x_{13}$. In this case the variable to eliminate is $x_3$ and the evaluation also includes the values in \texttt{value}[3] in addition to the functions in $F_{2}$, because $C_3$ is a child clique of $C_2$. Finally, the root clique $C_1$ is evaluated and all the $2^4$ possible combinations of variables $x_7$, $x_{12}$, $x_{13}$ and $x_{15}$ are evaluated using the subfunctions in $F_{1}$ and the array \texttt{value}[2] to find the objective value of an optimal offspring. The offspring itself is built using the arrays \texttt{variable}. In particular, \texttt{variable}[1] will store the combination of variables $x_7$, $x_{12}$, $x_{13}$ and $x_{15}$ that produces the best offspring. Array \texttt{variable}[2] will provide the value of $x_3$ given the ones of $x_7$ and $x_{13}$, which were provided by \texttt{variable}[1]. Array \texttt{variable}[3] will provide the value for $x_8$ given the value of $x_3$ provided by \texttt{variable}[2].

\begin{theorem}
Given two parent solutions $x$ and $y$ with differing set of variables $d(x,y)$ that produces clique tree $T$, Algorithm~\ref{alg:dynp} computes a best offspring $z$ in the largest dynastic potential of $x$ and $y$. That is:
\begin{equation}
\label{eqn:dynp}
f(z)=\max_{w \in \Bo^n \wedge 1_{d(x,y)}} f(x \oplus w).
\end{equation}
\end{theorem}
\begin{proof}
We will prove the theorem by structural induction over the clique tree. 
We will denote with $T_i$ the subtree of $T$ with $C_i$ in the root. We also introduce $C(T_i)$ as the union of the $C_j$ sets for all the cliques $C_j \in T_i$ and will use the convenient notation $R(T_i)=C(T_i)-S_i$. Observe that $C(T)=R(T)=d(x,y)$. Regarding the subfunctions, we introduce the notation $F_{T_i}$ to refer to the set of subfunctions associated with a clique in $T_i$: $F_{T_i} = \cup_{C_j \in T_i} F_j$. We only need to consider the subfunctions in $F_T$, because the remaining ones are constant in the dynastic potential of $x$ and $y$. Thus, Eq.~\eqref{eqn:dynp} is equivalent to:
\begin{equation}
\label{eqn:dynp-proof}
f(z)=\max_{w \in \Bo^n \wedge 1_{d(x,y)}} \sum_{f \in F_T} f(x \oplus w).
\end{equation}

The claim to be proven is that for each clique $C_i$, after its evaluation using Lines~\ref{lin:si-iteration} to~\ref{lin:si-iteration-end} of Algorithm~\ref{alg:dynp}, the array \texttt{value}[$i$] holds the following equation:
\begin{equation}
\label{eqn:induction}
\text{value}[i][w]=\max_{v \in \Bo^n \wedge 1_{R(T_i)}} \sum_{f \in F_{T_i}} f(x \oplus w \oplus v) \;\; \forall w \in \Bo^n \wedge 1_{S_i}. 
\end{equation}
Eq.~\eqref{eqn:induction} reduces to Eq.~\eqref{eqn:dynp-proof} when the clique $C_i$ is the root of $T$ and $T_i=T$. Thus, we only need to prove Eq.~\eqref{eqn:induction} using structural induction.
Let us start with the base case: a leaf clique. In this case, $R(T_i)=R_i$ and $C(T_i)=C_i$ and there is no child clique to iterate over in the for loop of Lines~\ref{lin:children-evaluation-start} to~\ref{lin:children-evaluation-end}. Eq.~\eqref{eqn:induction} is reduced to:
\begin{equation}
\label{eqn:base}
\text{value}[i][w]=\max_{v \in \Bo^n \wedge 1_{R_i}} \sum_{f \in F_{i}} f(x \oplus w \oplus v) \;\; \forall w \in \Bo^n \wedge 1_{S_i},
\end{equation}
and Lines~\ref{lin:children-evaluation-start} to~\ref{lin:children-evaluation-end} fill the \texttt{value}[$i$] array using exactly the expression in Eq.~\eqref{eqn:base}.

Now, we use the induction hypothesis to prove that Eq.~\eqref{eqn:induction} holds for any other node in the tree. In this case we have $C(T_i)=C_i \cup \bigcup_{j \in ch(i)} C(T_j)$ and $R(T_i)=R_i \cup \bigcup_{j \in ch(i)} R(T_j)$. The values computed by Lines~\ref{lin:children-evaluation-start} to~\ref{lin:children-evaluation-end} and stored in the \texttt{value}[$i$] array for all $w \in \Bo^n \wedge 1_{S_i}$ are:
\begin{equation}
\label{eqn:ind1}
\text{value}[i][w] = \max_{v \in \Bo^n \wedge 1_{R_i}} \left( \sum_{f \in F_{i}} f(x \oplus w \oplus v) +  \sum_{j \in ch(i)} \text{value}[j][(w \oplus v) \wedge 1_{S_j}] \right),
\end{equation}
and using the induction hypothesis we can replace \texttt{value}[$j$][$(w \oplus v) \wedge 1_{S_j}$] by the right hand side of Eq.~\eqref{eqn:induction} to write:
\begin{align}
\nonumber
\text{value}[i][w] = \max_{v \in \Bo^n \wedge 1_{R_i}} \left( \sum_{f \in F_{i}} f(x \oplus w \oplus v) +  \sum_{j \in ch(i)} \max_{v' \in \Bo^n \wedge 1_{R(T_j)}} \sum_{f \in F_{T_j}} f(x \oplus w \oplus v \oplus v') \right),
\end{align}
where we replaced $(w \oplus v) \wedge 1_{S_j}$ by $w \oplus v$ in the inner sum because the subfunctions in $F_{T_j}$ do not depend on any variable in $C_i-S_j$, which are the ones that differ in both expressions. The sets $F_{T_j}$ are disjoint for all $j \in ch(i)$, as well as the sets $R(T_j)$. Thus, we can swap the maximum and the sum to write:
\begin{align}
\nonumber
\text{value}[i][w] = \max_{v \in \Bo^n \wedge 1_{R_i}} \left( \sum_{f \in F_{i}} f(x \oplus w \oplus v) +  \max_{v' \in \Bo^n \wedge 1_{R(T_i)-R_i}} \sum_{f \in F_{T_i}-F_i} f(x \oplus w \oplus v \oplus v') \right),
\end{align}
where we used the identities $\bigcup_{j \in ch(i)} R(T_j) = R(T_i)-R_i$ and $\bigcup_{j \in ch(i)} F_{T_j} = F_{T_i} -F_i$ to simplify the expression. Finally, we can introduce the first sum in the maximum and notice that $v'$ is zero in the variables of $R_i$ to write:
\begin{align}
\nonumber
\text{value}[i][w] &= \max_{v \in \Bo^n \wedge 1_{R_i}} \max_{v' \in \Bo^n \wedge 1_{R(T_i)-R_i}} \sum_{f \in F_{T_i}} f(x \oplus w \oplus v \oplus v'),
\end{align}
which is Eq.~\eqref{eqn:induction} written in a different way.
\end{proof}

The operator described is an optimal recombination operator: it finds a best offspring from the largest dynastic potential.
The time required to evaluate one clique in Algorithm~\ref{alg:dynp} is $O((|F_{i}|+|ch(i)|) 2^{|C_i|})$. The number of children is bounded by $n$ and the number of subfunctions $m$ is bounded by $O(n^k)$ due to the $k$-bounded epistasis of $f$. However, the exponential factor is a threat to the efficiency of the algorithm. In the worst case $C_i$ can contain all the variables and the factor would be $2^n$. 

\subsection{Limiting the Complexity}
\label{subsec:complexity}

In order to avoid the exponential runtime of DPX, we propose to limit the exploration in Lines~\ref{lin:si-iteration} and~\ref{lin:ri-iteration} of Algorithm~\ref{alg:dynp}. Instead of iterating over all the possible combinations for all the variables in the separators $S_i$ and the residues $R_i$, we fix a bound $\beta$ on the number of variables that will be exhaustively explored. The remaining variables will jointly take only two values, each one coming from one of the parents. In a separator $S_i$ or residue $R_i$ with more than $\beta$ variables, we still have to decide which variables are exhaustively explored and which ones will be taken from one parent. One important constraint in this decision is that once we decide that two variables $x_i$ and $x_j$ will be taken from the same parent, then this should also happen in the other cliques where the two variables appear. We use a disjoint-set data structure~\citep{Tarjan1975} to keep track of the variables that must be taken from the same parent. In each clique, the variables that are articulation points in the VIG are added first to the list of variable to be fully explored. The motivation for this can be found in Section~\ref{subsec:theo-comparison}. The other variables are added in an arbitrary order (implementation dependent) to the list of variables to fully explore. We defer to future work a more in depth analysis on the strategies to decide which variables are fully explored in each clique.

Let us illustrate this with our example of Figures~\ref{fig:recom2} and~\ref{fig:mcs}, where we set $\beta=2$. This does not affect to the evaluation of $C_2$ or $C_3$, because in both cases the sets $S_i$ and $R_i$ have cardinality less than or equal to $\beta=2$, so all the variables in $R_2$, $S_2$, $R_3$ and $S_3$ will be fully explored. However, once cliques $C_3$ and $C_2$ (in that order) have been evaluated, we need to evaluate $C_1$ and, in this case, $|R_1|=4 > 2 = \beta$. For the evaluation of $C_1$, two variables, say $x_7$ and $x_{12}$, are fully enumerated (the four combinations of values for them are considered) and the other two variables, $x_{13}$ and $x_{15}$, are taken from the same parent and only two combinations are considered for them: 00 (red parent) and 11 (blue parent). In total, only $2^3=8$ combinations of values for the variables in this clique are explored, instead of the $2^4=16$ possible combinations.

This reduces the exponential part of the complexity of Algorithm~\ref{alg:dynp} to $2^{2(\beta+1)}$. Since $\beta$ is a predefined constant parameter decided by the user, the exponential factor turns into a constant. The operator is not anymore an optimal recombination operator. In the cases where $\beta+1 \geq \max \{|R_i|,|S_i|\}$ for all the cliques, DPX will still return the optimal offspring. 

\begin{theorem}
Given a function in the form of Equation~\eqref{eqn:fn} with $m$ subfunctions, the complexity of DPX with a constant bound $\beta$ for the number of exhaustively explored variables is $O(4^{\beta} (n+m)+n^2)$.
\end{theorem}
\begin{proof}
We have seen in Section~\ref{subsec:chordal} that the complexity of maximum cardinality search, the fill-in procedure and the clique tree construction is $O(n^2)$. The assignment of subfunctions to cliques can be done in $O(n+m)$ time, using the variable ordering found by MCS to assign the subfunctions that depends on each visited variable to the only clique where the variable is a residue. The complexity of the dynamic programming computation is:
\begin{align*}
O\left( \sum_{i} \left(|F_{i}|+|{ch}(i)|\right) 2^{2(\beta+1)}  \right) &= O\left( 4^{\beta} \left(m + \sum_{i} |{ch}(i)|\right)  \right) = O(4^{\beta}(m + n)), \\
\end{align*}
where we used the fact that the sum of the cardinality of the children for all the cliques is the number of edges in the clique tree, which is the number of cliques minus one, and the number of cliques is $O(n)$.
The reconstruction of the offspring solution requires to read all the \texttt{variable} arrays until building the solution. The complexity of this procedure is $O(n)$.
\end{proof}

In many cases, the number of subfunctions $m$ is $O(n)$ or $O(n^2)$. In these cases, the complexity of DPX reduces to $O(4^{\beta}n^2)$. But complexity can even reduce to $O(n)$ in some cases. In particular, when all the connected components of the recombination graph are paths or have a number of variables bounded by a constant, the number of edges in the original and the chordal graph is $O(n)$ and the complexity of DPX inherits this linear behaviour. 
This is the case for the recombination graph showed in Figure~\ref{fig:maxsat-recomb} for a real SAT instance, so this linear time complexity is not unusual, even in real and hard instances.

%

\subsection{Theoretical comparison with PX and APX}
\label{subsec:theo-comparison}

DPX is not worse than PX, since, in the worst case, it will pick the variables for each connected component of the recombination graph from one of the parent solutions (what PX does). In other words, if $\beta=0$ and there is only one clique in all the connected components of the recombination graph (worst case), DPX and PX behave the same and produce offspring with the same quality.
We wonder, however, if this happens with APX. If $\beta+1 \geq \max \{|R_i|,|S_i|\}$ for all the cliques $C_i$ in the chordal graph derived from the recombination graph, DPX cannot be worse than any recombination operator with the property of gene transmission and, in particular, it cannot be worse than APX. Otherwise, if the limit in the exploration explained in Section~\ref{subsec:complexity} is applied, it could happen that articulation points are not explored as they are in APX. One possible threat to the articulation points exploration in DPX is that they disappear after making the graph chordal. 
Fortunately, to make a graph chordal, the fill-in procedure only adds edges joining vertices in a cycle and, thus, it keeps the articulation points.
We provide formal proofs in the following. 

\begin{lemma}
\label{lem:fillin}
The fill-in procedure adds edges joining vertices in a cycle of the original graph.
\end{lemma}
\begin{proof}
Let us assume that edge $(v,w)$ is added to the graph by the fill-in procedure.
The values $\gamma(v)$ and $\gamma(w)$ are the numbers assigned by maximum cardinality search to nodes $v$ and $w$. Let us assume without loss of generality that $\gamma(v) < \gamma(w)$. 
The definition of fill-in \citep{Tarjan:Yannakis1984} implies that there is a path between $v$ and $w$ where all the intermediate nodes have a $\gamma$ value lower than $\gamma(v)$. 
On the other hand, during the application of maximum cardinality search the set of numbered nodes form a connected component of the graph. 
This implies that at the moment in which $v$ was numbered there existed a path between $v$ and $w$ with $\gamma$ values higher than $\gamma(v)$. As a consequence, two non-overlapping paths exist between $v$ and $w$ in the original graph and they form a cycle.
\end{proof}

\begin{theorem}
\label{thm:ap-fillin}
Articulation points of a graph are kept after the fill-in procedure.
\end{theorem}
\begin{proof}
According to Lemma~\ref{lem:fillin} all the edges added by the fill-in procedure join vertices in a cycle of the original graph. This means that the edges are added to bi-connected components of the graph, and never join vertices in two different bi-connected components. Adding edges to a bi-connected component never removes articulation points and the result follows.
\end{proof}

The previous theorem implies that articulation points of the original recombination graph are also articulation points of the chordal graph. Articulation points of a chordal graph are minimal separators of cardinality one \citep{Galinier1995} and they will appear in the sets $S_i$ of some cliques $C_i$. They are, thus, identified during the clique tree construction. 
This inspires a mechanism to reduce the probability that a solution explored in APX is not explored in DPX. 
In each clique $C_i$ when $\beta$ variables are chosen to be exhaustively explored (Lines~\ref{lin:si-iteration} and~\ref{lin:ri-iteration} of Algorithm~\ref{alg:dynp}) we choose the articulation points first. This way, articulation points can be exhaustively explored with higher probability. 
The only thing that can prevent articulation points from being explored is that many of them appear in one single clique. This situation is illustrated in Figure~\ref{fig:pathological-dpx}. For $\beta=0$ all the cliques are evaluated only in the two parent solutions as PX does, while APX explores 20 different combinations of variables according to Eq. (6) of \cite{Chicano2018gecco}.
For $\beta = 1$, the cliques $C_2$, $C_3$ and $C_4$ are fully explored, but 
in the clique of articulation points, $C_1$, only variable $x_4$ is fully enumerated, variables $x_5$ and $x_6$ are taken from the same parent. The total number of solutions explored is 32, which is more than the ones analyzed by APX (20), but the articulation points $x_5$ and $x_6$ are not explored in the same way as in APX and the set of solutions explored by DPX and APX differ. Thus, APX could find an offspring with higher fitness than the one obtained by DPX with $\beta=1$.

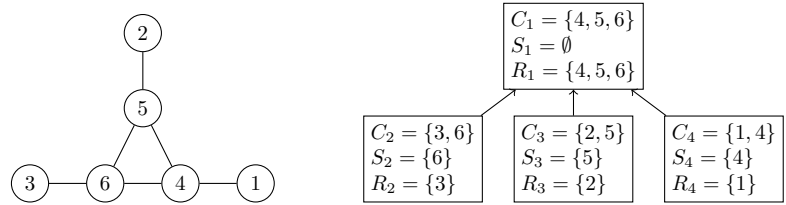
\begin{figure}[!ht]
\centering
\tikzstyle{estilo}=[circle,draw=black]
\tikzstyle{numbers}=[text=red]
\tikzstyle{clique}=[rectangle,draw=black,align=left]
%
\begin{tikzpicture}[scale=0.50, every node/.style={scale=0.8}]
\node[estilo] (x3) at (0,0) {$3$};
\node[estilo] (x6) at (2,0) {$6$};
\node[estilo] (x4) at (4,0) {$4$};
\node[estilo] (x1) at (6,0) {$1$};
\node[estilo] (x5) at (3,2) {$5$};
\node[estilo] (x2) at (3,4) {$2$};
\draw (x3) -- (x6);
\draw (x6) -- (x4);
\draw (x4) -- (x1);
\draw (x6) -- (x5);
\draw (x4) -- (x5);
\draw (x2) -- (x5);
\end{tikzpicture} \hspace{1cm}
\begin{tikzpicture}[scale=0.50, every node/.style={scale=0.8}]
\node[clique] (c1) at (4,3) {$C_1= \{4,5,6\}$\\$S_1=\emptyset$\\$R_1=\{4,5,6\}$};
\node[clique] (c2) at (0,0) {$C_2= \{3,6\}$\\$S_2=\{6\}$\\$R_2=\{3\}$};
\node[clique] (c3) at (4,0) {$C_3= \{2,5\}$\\$S_3=\{5\}$\\$R_3=\{2\}$};
\node[clique] (c4) at (8,0) {$C_4= \{1,4\}$\\$S_4=\{4\}$\\$R_4=\{1\}$};
\draw [->] (c2) -- (c1);
\draw [->] (c3) -- (c1);
\draw [->] (c4) -- (c1);
\end{tikzpicture}

\caption{Connected component in a recombination graph (left) and its clique tree (right). DPX with \mbox{$\beta = 1$} explores the articulation points in a different way as APX.}
\label{fig:pathological-dpx}
\end{figure}

\subsection{Generalization of DPX}
Although this paper focuses on Mk Landscapes, defined over binary variables, DPX can also be applied as is when the variables take their values in a finite set different from the binary set. In this case, one can imagine that 0 represents the value of a differing variable in the red parent, and 1 the value of the same variable in the blue parent. All the results, including runtime guarantees, are the same. The only difference is that the offspring of DPX is not optimal, in general, in the smallest hyperplane containing the parent solutions, because an optimal solution could have values not present in the parents. Even in this case, DPX can be modified to keep the same runtime and provide an optimal solution in the mentioned hyperplane, at the only cost of violating the gene transmission property.

\section{Experiments}
\label{sec:experiments}

This section will focus on the experimental evaluation of DPX in comparison with other previous crossover operators. In particular, we want to answer the following two research questions:
\begin{itemize}
\item RQ1: How does DPX perform compared to other crossover operators in terms of runtime and quality of offspring?
\item RQ2: How does DPX perform included in a search algorithm for solving NP-hard pseudo-Boolean optimization problems?
\end{itemize}

Regarding the other operators in the comparison, we include PX and APX because they are gray box crossover operators using the VIG and we want to check the claims exposed in Section~\ref{subsec:theo-comparison} that relate DPX with these two operators.
We also include in this comparison two other operators which dynastic potential has the same size as DPX ($2^{h(x,y)}$): \textit{uniform crossover} and \textit{network crossover}.
In uniform crossover (UX), each variable is taken from one of the two parents with probability 0.5. Network crossover (NX)~\citep{Hauschild2010} uses the learned linkages among the variables to select groups of variables from one parent such that variables from the same connected component in the linkage graph are selected first. In our case, we have complete knowledge of the linkage graph: the variable interaction graph. The VIG is used in our implementation of network crossover and variables are selected using randomized breadth first search in the VIG and starting from a random variable until half of the variables are selected. Then, the group of selected variables is taken from one of the parents and inserted into the other to form the offspring.

Two different kinds of NP-hard problems are used in the experiments: NKQ Landscapes, an academic benchmark which allows us to parameterize the density of edges in the VIG by changing $K$, and MAX-SAT instances from the MAX-SAT Evaluation 2017\footnote{ \url{http://mse17.cs.helsinki.fi/benchmarks.html}.}.
Random NKQ (``Quantized'' NK) landscapes~\citep{Newman1998} can be seen as Mk Landscapes with one subfunction per variable ($m=n$). Each subfunction $f_{\ell}$ depends on variable $x_{\ell}$ and other $K$ random variables, and the codomain of each subfunction is the set $\{0,1,\ldots, Q-1\}$, where $Q$ is a positive integer. Thus, each subfunction depends on exactly $k=K+1$ variables.
The values of the subfunctions are randomly generated, that is, for each subfunction and each combination of variables in the subfunction an integer value in the interval $[0,Q-1]$ is randomly selected following a uniform distribution.
Random NKQ landscapes are NP-hard when $K\geq 2$. The parameter $K$ determines the higher order nonzero Walsh coefficients in its Walsh decomposition, which is a measure of the ``ruggedness'' of the landscape~\citep{Hordijk:1998ej}.
Regarding MAX-SAT, we used the same instances as~\cite{Chicano2018gecco}\footnote{The list of instances is available together with the source code of DPX in Github.} to allow the comparison with APX. They are 160 unweighted and 132 weighted instances. 

The computer used for the experiments is one multicore machine with four Intel Xeon CPU (E5-2680 v3) at 2.5~GHz, summing a total of 48 cores, 192 GB of memory and Ubuntu 16.04 LTS. The source code of all the algorithms and operators used in the experiments can be found in Github\footnote{\url{https://github.com/jfrchicanog/EfficientHillClimbers}} including a link to a docker image to ease the reproduction of the experimental results.

Section~\ref{subsec:crossover-comp} answers RQ1 and Section~\ref{subsec:algorithms-comp} answers RQ2. In Section~\ref{subsec:lon} we include a local optima network analysis of the best overall algorithm identified in Section~\ref{subsec:algorithms-comp} to better understand its behaviour.

\subsection{Crossover comparison}
\label{subsec:crossover-comp}

This section will present the experiments to answer RQ1: how does DPX perform compared to APX, PX, NX and UX in terms of runtime and quality of offspring? In the case of DPX we use values for $\beta$ from 0 to 5. The optimization problem used is random NKQ Landscapes with $n=\numprint{10000}$  variables, $K=2,3,4,5$ and $Q=64$. For each value of $K$ we generated ten different instances, summing a total of 40 NKQ Landscapes instances. In each of them we randomly generated \numprint{6000} pairs of solutions with different Hamming distance between them and applied all the crossover operators. Six different values of Hamming distance $h$ were used, generating \numprint{1000} pairs of random solutions for each Hamming distance. Expressed in terms of the percentage of differing variables, the values for $h$ are 1\%, 2\%, 4\%, 8\%, 16\% and 32\%.
Two metrics were collected in each application of all the crossover operators: runtime and quality improvement over the parents. The crossover runtime was measured with nanoseconds precision (expressed in the tables in an appropriate multiple) and the quality of the offspring is expressed with a relative measure of quality improvement. If $x$ and $y$ are the parent solutions and $z$ is the offspring we define the \emph{quality improvement ratio} (QIR) in a maximization problem as:
\begin{equation}
QIR_f(x,y,z) = \frac{f(z)-\max\{f(x),f(y)\}}{\max\{f(x),f(y)\}},
\end{equation}
that is, the fraction of improvement of the offspring compared to the best parent. All the experiments were run with a memory limit of 5GB of RAM. In the case of PX, APX and DPX we also collected the number of implicitly explored solutions, expressed with its logarithm (the offspring is one best solution in the set of implicitly explored solutions) and the fraction of runs in which the crossover behaves like an optimal recombination (returns the best solution in the largest dynastic potential). Tables~\ref{tab:crossover-comp-runtime-10K} to~\ref{tab:crossover-comp-full-10K} present the runtime, quality of improvement, logarithm of explored solutions and percentage of crossover runs where an optimal offspring is returned. The figures are averages over \numprint{10000} samples (\numprint{1000} crossover operations in each of the ten instances for each value of $K$).

\begin{table}[!ht]
\caption{Average runtime of crossover operators for random NKQ Landscapes with $n=$\numprint{10000} variables. Time is in microseconds ($\mu$s) for UX and in milliseconds (ms) for the rest. The Hamming distance between parents, $h$, is expressed in percentage of variables.}
\label{tab:crossover-comp-runtime-10K}
\centering
\begin{tabular}{@{}rrrrrrrrrrr@{}}
\toprule
 \multicolumn{1}{c}{$h$} & \multicolumn{1}{c}{UX} & \multicolumn{1}{c}{NX} & \multicolumn{1}{c}{PX} & \multicolumn{1}{c}{APX} & \multicolumn{6}{c}{DPX (ms)} \\
\cmidrule{6-11}
 \multicolumn{1}{c}{\%} &  \multicolumn{1}{c}{$\mu$s} & \multicolumn{1}{c}{ms} & \multicolumn{1}{c}{ms} & \multicolumn{1}{c}{ms} & $\beta=0$ & $\beta=1$ & $\beta=2$ & $\beta=3$ & $\beta=4$ & $\beta=5$ \\
\midrule
\multicolumn{11}{c}{$K=2$} \\
1 & 73 & 1.2 & 0.5 & 1.0 & 0.8 & 0.9 & 0.8 & 0.8 & 0.8 & 0.9 \\
2 & 95 & 2.3 & 0.9 & 2.5 & 2.1 & 2.3 & 2.4 & 2.0 & 2.1 & 1.9 \\
4 & 93 & 2.3 & 1.4 & 4.5 & 2.9 & 2.8 & 2.9 & 2.5 & 2.5 & 2.4 \\
8 & 120 & 2.3 & 2.2 & 7.2 & 6.3 & 6.9 & 6.3 & 5.8 & 5.8 & 5.7 \\
16 & 113 & 1.2 & 2.8 & 7.1 & 5.5 & 5.9 & 5.8 & 5.8 & 5.4 & 5.3 \\
32 & 154 & 1.7 & 9.3 & 12.7 & 22.1 & 22.8 & 23.5 & 23.3 & 24.6 & 23.3 \\
\multicolumn{11}{c}{$K=3$} \\
1 & 92 & 1.7 & 0.6 & 1.5 & 1.0 & 1.0 & 1.0 & 1.0 & 0.9 & 1.0 \\
2 & 87 & 2.4 & 1.2 & 3.5 & 1.8 & 1.6 & 2.0 & 1.7 & 1.6 & 1.7 \\
4 & 97 & 2.8 & 1.9 & 6.3 & 3.1 & 3.1 & 3.1 & 2.7 & 2.4 & 2.7 \\
8 & 125 & 2.8 & 3.0 & 8.3 & 4.7 & 4.9 & 5.7 & 5.6 & 4.5 & 4.9 \\
16 & 120 & 1.9 & 5.1 & 9.1 & 9.7 & 10.5 & 10.1 & 10.9 & 11.0 & 11.3 \\
32 & 143 & 2.2 & 5.9 & 16.0 & 251.4 & 257.5 & 256.4 & 273.1 & 267.5 & 263.8 \\
\multicolumn{11}{c}{$K=4$} \\
1 & 79 & 2.8 & 0.9 & 1.9 & 1.1 & 1.1 & 1.3 & 1.4 & 1.0 & 1.2 \\
2 & 96 & 3.8 & 1.7 & 4.4 & 1.7 & 1.9 & 2.1 & 1.8 & 1.5 & 1.6 \\
4 & 93 & 3.4 & 2.2 & 7.1 & 3.2 & 3.3 & 3.5 & 3.6 & 3.2 & 3.5 \\
8 & 99 & 3.5 & 4.8 & 11.6 & 5.6 & 5.7 & 5.4 & 6.0 & 6.8 & 7.0 \\
16 & 116 & 2.7 & 3.7 & 11.2 & 31.7 & 31.9 & 32.8 & 33.3 & 34.0 & 36.2 \\
32 & 143 & 2.8 & 5.2 & 18.0 & 596.7 & 601.9 & 587.8 & 598.9 & 683.0 & 692.4 \\
\multicolumn{11}{c}{$K=5$} \\
1 & 68 & 3.2 & 0.9 & 2.6 & 1.4 & 1.5 & 1.5 & 1.4 & 1.4 & 1.3 \\
2 & 82 & 3.7 & 1.8 & 5.2 & 2.1 & 2.3 & 2.3 & 2.1 & 2.2 & 2.0 \\
4 & 85 & 4.2 & 3.5 & 8.7 & 3.6 & 3.9 & 3.8 & 3.9 & 4.0 & 4.1 \\
8 & 119 & 4.3 & 5.4 & 13.3 & 8.0 & 8.1 & 8.2 & 9.5 & 10.9 & 9.9 \\
16 & 113 & 3.0 & 4.1 & 12.8 & 90.7 & 83.0 & 103.0 & 92.2 & 101.3 & 107.5 \\
32 & 139 & 3.7 & 5.8 & 19.4 & \numprint{1000.5} & \numprint{1034.0} & \numprint{1041.1} & \numprint{1020.3} & \numprint{1089.9} & \numprint{1021.7} \\
\bottomrule					
\end{tabular}
\end{table}

Regarding the runtime, 
we observe some clear trends that we will comment in the following. Uniform crossover is the fastest algorithm (less than 200$\mu$s in all the cases). It randomly selects one parent for each differing bit and this can be done very fast. The other operators are based on the VIG and they require more time to explore it and compute the offspring. Their runtime can be best measured in milliseconds. NX, PX and APX have runtimes between less than one millisecond to 20 ms. DPX is clearly the slowest crossover operator when the parent solutions differ in 32\% of the bits (\numprint{3200} variables), reaching 1 second of computation for instances with $K=5$. For lower values of $h$, APX is sometimes slower than DPX. We also observe an increase in runtime with $h$, which can be explained because the recombination graph to explore is larger. It is interesting to note that no exponential increase is observed in the runtime when $\beta$ increases linearly. To explain this we have to observe Tables~\ref{tab:crossover-comp-logarithm-10K} and~\ref{tab:crossover-comp-full-10K}, where we can see that DPX is able to completely explore the dynastic potential when $h$ is low and the logarithm of the  number of solutions explored increase very slowly with $\beta$ because it is near the maximum possible.
High runtime is one of the drawbacks of DPX, and the other one is memory consumption. For the instances with $n=\numprint{10000}$ variables, 5GB of memory is enough when $K \leq 5$, but we run some experiments with $n=\numprint{100000}$ in which DPX ended with an ``Out of Memory'' error. In these cases, increasing the memory could help, but the amount of memory required is higher than that required by PX and APX, and much higher than the memory required by UX and NX.

\begin{table}[!ht]
\caption{Average quality improvement ratio of crossover operators for random NKQ Landscapes with $n=$\numprint{10000} variables. The numbers are in parts per thousand (\textperthousand). The Hamming distance between parents, $h$, is expressed in percentage of variables.}
\label{tab:crossover-comp-quality-10K}
\centering
\setlength{\tabcolsep}{5.5pt}
\begin{tabular}{@{}rrrrrrrrrrrr@{}}
\toprule
 \multicolumn{1}{c}{$h$} & \multicolumn{1}{c}{UX} & \multicolumn{1}{c}{NX} & \multicolumn{1}{c}{PX} & \multicolumn{1}{c}{APX} & \multicolumn{6}{c}{DPX (\textperthousand)} \\
\cmidrule{6-11}
 \multicolumn{1}{c}{\%} & \multicolumn{1}{c}{\textperthousand} & \multicolumn{1}{c}{\textperthousand} & \multicolumn{1}{c}{\textperthousand} & \multicolumn{1}{c}{\textperthousand} & $\beta=0$ & $\beta=1$ & $\beta=2$ & $\beta=3$ & $\beta=4$ & $\beta=5$ \\
\midrule
\multicolumn{11}{c}{$K=2$} \\
1  & {-0.58} & {-0.55} & {4.92} & {4.93} & {4.92} & {5.04} & {5.04} & {5.04} & {5.04} & {5.04} \\
2  & {-0.79} & {-0.81} & {9.89} & {9.99} & {9.95} & {10.38} & {10.39} & {10.39} & {10.39} & {10.39} \\
4  & {-1.13} & {-1.11} & {19.28} & {19.96} & {19.70} & {21.21} & {21.23} & {21.23} & {21.23} & {21.23} \\
8  & {-1.56} & {-1.54} & {35.04} & {39.19} & {38.15} & {42.80} & {42.92} & {42.92} & {42.92} & {42.92} \\
16 & {-2.08} & {-2.07} & {53.43} & {70.87} & {75.03} & {85.72} & {86.21} & {86.21} & {86.21} & {86.21} \\
32 & {-2.72} & {-2.71} & {34.41} & {42.09} & {108.86} & {123.98} & {134.38} & {137.29} & {138.78} & {139.76} \\
\multicolumn{11}{c}{$K=3$} \\
1  & {-0.64} & {-0.65} & {5.57} & {5.62} & {5.60} & {5.84} & {5.84} & {5.84} & {5.84} & {5.84} \\
2  & {-0.92} & {-0.91} & {10.93} & {11.33} & {11.18} & {12.02} & {12.03} & {12.03} & {12.03} & {12.03} \\
4  & {-1.29} & {-1.26} & {20.10} & {22.50} & {21.95} & {24.50} & {24.57} & {24.57} & {24.57} & {24.57} \\
8  & {-1.72} & {-1.77} & {30.80} & {40.40} & {43.67} & {49.37} & {49.66} & {49.66} & {49.66} & {49.66} \\
16 & {-2.39} & {-2.37} & {21.30} & {25.45} & {63.04} & {70.96} & {77.15} & {79.14} & {80.21} & {80.96} \\
32 & {-2.85} & {-2.85} & {6.55} & {7.38} & {59.15} & {63.68} & {76.87} & {83.34} & {86.36} & {88.23} \\
\multicolumn{11}{c}{$K=4$} \\
1  & {-0.74} & {-0.74} & {6.02} & {6.18} & {6.12} & {6.51} & {6.52} & {6.52} & {6.52} & {6.52} \\
2  & {-1.04} & {-1.04} & {11.47} & {12.48} & {12.20} & {13.46} & {13.48} & {13.48} & {13.48} & {13.48} \\
4  & {-1.42} & {-1.40} & {18.98} & {23.81} & {24.25} & {27.38} & {27.50} & {27.50} & {27.50} & {27.50} \\
8  & {-1.92} & {-1.95} & {17.30} & {21.78} & {41.39} & {46.48} & {49.29} & {50.46} & {51.32} & {52.04} \\
16 & {-2.47} & {-2.55} & {6.92} & {7.93} & {41.63} & {45.38} & {53.90} & {57.24} & {58.91} & {59.98} \\
32 & {-3.15} & {-3.13} & {1.35} & {1.95} & {40.98} & {42.14} & {46.88} & {53.52} & {58.04} & {60.35} \\
\multicolumn{11}{c}{$K=5$} \\
1  & {-0.79} & {-0.78} & {6.38} & {6.72} & {6.61} & {7.18} & {7.18} & {7.18} & {7.18} & {7.18} \\
2  & {-1.10} & {-1.10} & {11.46} & {13.40} & {13.17} & {14.77} & {14.81} & {14.81} & {14.81} & {14.81} \\
4  & {-1.53} & {-1.56} & {15.06} & {20.38} & {26.44} & {29.58} & {30.06} & {30.14} & {30.16} & {30.17} \\
8  & {-2.07} & {-2.06} & {8.07} & {9.56} & {31.18} & {34.54} & {39.26} & {41.02} & {41.98} & {42.67} \\
16 & {-2.68} & {-2.66} & {2.19} & {2.90} & {30.14} & {31.61} & {37.08} & {41.51} & {43.48} & {44.83} \\
32 & {-3.15} & {-3.13} & {0.28} & {0.77} & {32.42} & {32.82} & {34.18} & {36.64} & {40.31} & {44.05} \\
\bottomrule					
\end{tabular}
\end{table}

We can observe that the quality improvement ratio (Table~\ref{tab:crossover-comp-quality-10K}) is always positive in PX, APX and DPX. These three operators, by design, cannot provide a solution that is worse than the best parent. We also observe how the quality improvement ratio is always the highest for DPX. APX and PX are the second and third operators regarding this metric, respectively. The worst operators are UX and NX. They always show a negative quality improvement ratio. We can explain this with the following intuition: the expected fitness of the offspring is similar to the one of a random solution and the probability of improving both parents is 1/4 because the probability of improving each one is 1/2, which means that in most of the cases (3/4 of the cases on average) the offspring will be worse than the best parent. 
We can support this with some theory. Parents $x$ and $y$ are random solutions and their fitness values, $f(x)$ and $f(y)$, are random variables with unknown equal distribution.\footnote{In general, they are not independent because the Hamming distance among them is fixed. But the marginal distributions must be the same.} In UX and NX, the child $z$ is a random solution in the dynastic potential and the expectation of the random variable $f(z)$ should not differ too much from the one of $f(x)$ and $f(y)$ because the dynastic potential is large (at least $2^{100}$ in our experiments) and NKQ Landscapes are composed of functions randomly generated. 
We can use the following equations for the random variables,
\begin{eqnarray}
& & \max\{f(x),f(y)\} + \min\{f(x),f(y)\} = f(x) + f(y), \\
& & \max\{f(x),f(y)\} - \min\{f(x),f(y)\} = \left|f(x) - f(y)\right|,
\end{eqnarray}
where we can take the expectation at both sides and add them to get:
\begin{equation}
2 E[\max\{f(x),f(y)\}] = E[f(x)] + E[f(y)] + E [\left|f(x) - f(y)\right|].
\end{equation}
Finally, we use the fact that $E[f(x)] = E[f(y)]$ and the assumption that $E[f(z)] = E[f(x)]$ to write:
\begin{equation}
E[f(z) - \max\{f(x),f(y)\}] =  - E [\left|f(x) - f(y)\right|] / 2 < 0.
\end{equation}

\begin{table}[!ht]
\caption{Average logarithm in base 2 of the solutions implicitly explored by PX, APX and DPX for random NKQ Landscapes with $n=$\numprint{10000} variables. The Hamming distance between parents, $h$, is expressed in percentage of variables.}
\label{tab:crossover-comp-logarithm-10K}
\centering
\begin{tabular}{@{}rrrrrrrrrr@{}}
\toprule
\multicolumn{1}{c}{$h$}  & \multicolumn{1}{c}{PX} & \multicolumn{1}{c}{APX} & \multicolumn{6}{c}{DPX ($\log_2$)} \\
\cmidrule{4-9}
 \multicolumn{1}{c}{\%} & \multicolumn{1}{c}{$\log_2$} & \multicolumn{1}{c}{$\log_2$} & $\beta=0$ & $\beta=1$ & $\beta=2$ & $\beta=3$ & $\beta=4$ & $\beta=5$ \\
\midrule
\multicolumn{9}{c}{$K=2$} \\
1  & {97.1} & {97.3} & {97.2} & {100.0} & {100.0} & {100.0} & {100.0} & {100.0} \\
2  & {188.1} & {190.3} & {189.3} & {199.9} & {200.0} & {200.0} & {200.0} & {200.0} \\
4  & {352.9} & {368.1} & {362.0} & {399.5} & {400.0} & {400.0} & {400.0} & {400.0} \\
8  & {613.5} & {703.3} & {679.7} & {796.8} & {800.0} & {800.0} & {800.0} & {800.0} \\
16 & {873.6} & \numprint{1220.6} & \numprint{1311.2} & \numprint{1586.5} & \numprint{1600.0} & \numprint{1600.0} & \numprint{1600.0} & \numprint{1600.0} \\
32 & {660.7} & {828.3} & \numprint{2055.6} & \numprint{2399.2} & \numprint{2586.9} & \numprint{2636.5} & \numprint{2661.3} & \numprint{2677.4} \\
\multicolumn{9}{c}{$K=3$} \\
1  & {94.1} & {95.2} & {94.7} & {100.0} & {100.0} & {100.0} & {100.0} & {100.0} \\
2  & {176.5} & {184.3} & {181.2} & {199.8} & {200.0} & {200.0} & {200.0} & {200.0} \\
4  & {306.6} & {352.3} & {341.2} & {398.4} & {400.0} & {400.0} & {400.0} & {400.0} \\
8  & {437.3} & {602.5} & {663.0} & {793.2} & {799.9} & {800.0} & {800.0} & {800.0} \\
16 & {351.5} & {426.4} & \numprint{1019.5} & \numprint{1174.7} & \numprint{1271.0} & \numprint{1300.4} & \numprint{1316.0} & \numprint{1326.9} \\
32 & {141.5} & {155.1} & \numprint{1099.0} & \numprint{1179.1} & \numprint{1395.2} & \numprint{1499.5} & \numprint{1547.6} & \numprint{1576.8} \\
\multicolumn{9}{c}{$K=4$} \\
1  & {90.2} & {93.0} & {91.9} & {99.9} & {100.0} & {100.0} & {100.0} & {100.0} \\
2  & {161.2} & {179.0} & {173.7} & {199.5} & {200.0} & {200.0} & {200.0} & {200.0} \\
4  & {247.0} & {324.7} & {332.8} & {397.4} & {400.0} & {400.0} & {400.0} & {400.0} \\
8  & {238.2} & {305.8} & {580.8} & {674.0} & {713.9} & {729.6} & {740.9} & {750.4} \\
16 & {119.7} & {134.1} & {651.3} & {710.3} & {831.5} & {878.1} & {901.3} & {915.9} \\
32 & {31.1} & {39.5} & {719.8} & {737.4} & {812.0} & {914.8} & {983.3} & \numprint{1018.2} \\
\multicolumn{9}{c}{$K=5$} \\
1  & {85.4} & {91.1} & {89.1} & {99.9} & {100.0} & {100.0} & {100.0} & {100.0} \\
2  & {142.0} & {172.4} & {168.5} & {199.2} & {200.0} & {200.0} & {200.0} & {200.0} \\
4  & {175.4} & {246.1} & {332.2} & {390.6} & {398.2} & {399.4} & {399.8} & {399.9} \\
8  & {113.2} & {132.7} & {420.5} & {470.8} & {530.8} & {552.6} & {564.4} & {572.8} \\
16 & {38.9} & {47.6} & {449.0} & {469.0} & {542.8} & {601.9} & {627.7} & {645.3} \\
32 & {7.5} & {13.7} & {534.0} & {539.3} & {559.3} & {595.7} & {649.7} & {703.5} \\
\bottomrule					
\end{tabular}
\end{table}

In Table~\ref{tab:crossover-comp-logarithm-10K} we can observe how DPX explores more solutions than PX and APX when $\beta >0$. We also observe how APX can outperform DPX in the number of explored solutions when $\beta=0$, as we illustrated with an example in Section~\ref{subsec:theo-comparison}. The latter happens when $h \leq 800$ for $K=2$, when $h\leq 400$ for $K=3$ and when $h\leq 200$ for $K=4$ and $K=5$. DPX always explores more solutions than PX independently of the value of $\beta$, as the theory predicts. As $h$ grows, the dynastic potential increases and also the logarithm of explored solutions for DPX. In fact, in many cases we observe that this logarithm reaches in DPX the maximum possible value, $h$ (the Hamming distance between the parents). This is also reflected in Table~\ref{tab:crossover-comp-full-10K}, where we show the percentage of runs where the full dynastic potential is explored by the operators or, equivalently, the percentage of runs in which the logarithm of explored solutions is $h$. 
In the case of PX and APX the increase in $h$ does not always imply an increase in the number of explored solutions: there is a value of $h$ for which the logarithm of explored solutions reaches a maximum and then decreases with $h$. The number of explored solutions in these two operators is proportional to the number of connected components in the recombination graph. Starting from an empty recombination graph the number of connected components increases as new random variables are added, and this explains why the logarithm of explored solutions in PX and APX increases with $h$ for low values of $h$. At some critical value of $h$, the number of connected components starts to decrease because the new variables in the recombination graph join connected components, instead of generating new ones. The exact value of $h$ at which this happens is approximately $n/(K+1)$ for the adjacent NKQ Landscapes \citep{ChicanoWOT17}. It is difficult to compute this value for the random NKQ Landscapes that we use here, but the critical value must be a decreasing function of $K$. This dependence of the critical value with $K$ can also be observed in Table~\ref{tab:crossover-comp-logarithm-10K}: the value of $h$ at which the number of explored solutions is maximum decreases from $h=\numprint{1600}$ to $h=400$ when $K$ increases from 2 to 5.

Regarding the fraction of runs in which full dynastic potential exploration is achieved, PX and DPX with $\beta=0$ behaves the same and achieve full exploration for some pairs of parents only when $K\leq 3$. APX is slightly better and DPX is the best when $\beta \geq 1$, behaving like an optimal recombination operator in most of the cases for $K=2$ and $h\leq \numprint{1600}$. The fraction of runs with full dynastic potential exploration decreases when $K$ and $h$ increase, due to the higher number of edges in the variable interaction graph, which induces larger cliques.

\begin{table}[!ht]
\caption{Fraction of crossover runs where the full dynastic potential of the parent solutions is explored in PX, APX and DPX for random NKQ Landscapes with $n=$\numprint{10000} variables. The numbers are in percentages (\%). The Hamming distance between parents, $h$, is expressed in percentage of variables.}
\label{tab:crossover-comp-full-10K}
\centering
\begin{tabular}{@{}rrrrrrrrrr@{}}
\toprule
\multicolumn{1}{c}{$h$}  & \multicolumn{1}{c}{PX} & \multicolumn{1}{c}{APX} & \multicolumn{6}{c}{DPX(\%)} \\
\cmidrule{4-9}
 \multicolumn{1}{c}{\%} &  \multicolumn{1}{c}{\%} &  \multicolumn{1}{c}{\%} & $\beta=0$ & $\beta=1$ & $\beta=2$ & $\beta=3$ & $\beta=4$ & $\beta=5$ \\
\midrule
\multicolumn{9}{c}{$K=2$} \\
1  & {5.61} & {6.45} & {5.61} & {99.07} & {100.00} & {100.00} & {100.00} & {100.00} \\
2  & {0.00} & {0.00} & {0.00} & {93.17} & {100.00} & {100.00} & {100.00} & {100.00} \\
4  & {0.00} & {0.00} & {0.00} & {60.73} & {100.00} & {100.00} & {100.00} & {100.00} \\
8  & {0.00} & {0.00} & {0.00} & {4.11} & {99.98} & {100.00} & {100.00} & {100.00} \\
16 & {0.00} & {0.00} & {0.00} & {0.00} & {99.67} & {99.98} & {100.00} & {100.00} \\
32 & {0.00} & {0.00} & {0.00} & {0.00} & {0.00} & {0.00} & {0.00} & {0.00} \\
\multicolumn{9}{c}{$K=3$} \\
1  & {0.27} & {0.47} & {0.27} & {96.46} & {99.99} & {100.00} & {100.00} & {100.00} \\
2  & {0.00} & {0.00} & {0.00} & {77.88} & {99.94} & {100.00} & {100.00} & {100.00} \\
4  & {0.00} & {0.00} & {0.00} & {20.97} & {98.39} & {100.00} & {100.00} & {100.00} \\
8  & {0.00} & {0.00} & {0.00} & {0.22} & {88.87} & {99.93} & {99.99} & {100.00} \\
16 & {0.00} & {0.00} & {0.00} & {0.00} & {0.00} & {0.00} & {0.00} & {0.00} \\
32 & {0.00} & {0.00} & {0.00} & {0.00} & {0.00} & {0.00} & {0.00} & {0.00} \\
\multicolumn{9}{c}{$K=4$} \\
1  & {0.00} & {0.02} & {0.00} & {92.11} & {99.95} & {100.00} & {100.00} & {100.00} \\
2  & {0.00} & {0.00} & {0.00} & {59.18} & {99.54} & {100.00} & {100.00} & {100.00} \\
4  & {0.00} & {0.00} & {0.00} & {8.21} & {95.38} & {99.99} & {100.00} & {100.00} \\
8  & {0.00} & {0.00} & {0.00} & {0.00} & {0.06} & {0.27} & {1.00} & {2.20} \\
16 & {0.00} & {0.00} & {0.00} & {0.00} & {0.00} & {0.00} & {0.00} & {0.00} \\
32 & {0.00} & {0.00} & {0.00} & {0.00} & {0.00} & {0.00} & {0.00} & {0.00} \\
\multicolumn{9}{c}{$K=5$} \\
1  & {0.00} & {0.00} & {0.00} & {86.88} & {99.90} & {99.99} & {100.00} & {100.00} \\
2  & {0.00} & {0.00} & {0.00} & {43.67} & {98.89} & {99.98} & {100.00} & {100.00} \\
4  & {0.00} & {0.00} & {0.00} & {2.96} & {72.71} & {90.86} & {96.30} & {98.58} \\
8  & {0.00} & {0.00} & {0.00} & {0.00} & {0.00} & {0.00} & {0.00} & {0.00} \\
16 & {0.00} & {0.00} & {0.00} & {0.00} & {0.00} & {0.00} & {0.00} & {0.00} \\
32 & {0.00} & {0.00} & {0.00} & {0.00} & {0.00} & {0.00} & {0.00} & {0.00} \\
\bottomrule					
\end{tabular}
\end{table}

\subsection{Crossover in search algorithms}
\label{subsec:algorithms-comp}

In this section we want to answer RQ2: how does DPX perform when it is included in a search algorithm compared to APX, PX, NX and UX? From Section~\ref{subsec:crossover-comp} we conclude that DPX is better than PX, APX, UX and NX in terms of quality of the solution due to its higher exploration capability. This feature is not necessarily good when the operator is included in a search algorithm, because it can produce premature convergence. DPX is also slower, in general, than the other crossover operators in our experimental setting, and this could slow down the search. 
In order to answer the research question we included the crossover operators in two metaheuristic algorithms, each of them from a different family: deterministic recombination and iterated local search (DRILS), a trajectory-based metaheuristic; and an evolutionary algorithm (EA), a population-based metaheuristic. 

DRILS~\citep{ChicanoWOT17} uses a first improving move hill climber to reach a local optimum. The neighborhood in the hill climber is the set of solutions at Hamming distance one from the current solution. Then, it perturbs the solution by randomly flipping $\alpha n$ bits, where $\alpha$ is the so-called \emph{perturbation factor}, the only parameter of DRILS to be tuned. It then applies local search to the new solution to reach another local optimum and applies crossover to the last two local optima, generating a new solution that is improved further with the hill climber. This process is repeated until the stop condition is satisfied. The pseudocode of DRILS can be found in Algorithm~\ref{alg:drils}.

\begin{algorithm}[!t]
\caption{Deterministic recombination and iterated local search (DRILS)}
\label{alg:drils}
\KwData{$f$, stop condition, $\alpha$, crossover operator (and parameters)}
\KwResult{best (found solution)}
$x \leftarrow$ hillClimber(random())\;
best $\leftarrow x$\;
\While{not stopCondition()}{
	$y \leftarrow$ hillClimber (perturb($x$, $\alpha$))\;
	$z$ $\leftarrow$ crossover($x$, $y$)\;
	\label{lin:crossover}
	\eIf{$z = x$ or $z = y$}{
		$x \leftarrow y$\;
	}{
		$x \leftarrow$ hillClimber($z$)\;
		\If{$f(x) > f(\text{best})$}{
			best $\leftarrow x$\;
		}
	}
}
\Return best
\end{algorithm}

The evolutionary algorithm (EA) which we use in our empirical evaluation, is a steady-state genetic algorithm where two parents are selected and recombined to produce an offspring that is mutated. The mutated solution replaces the worst solution in the population only if its fitness is strictly higher. The mutation operator used is bit flip with probability $p_m$ of independently flipping each bit. Three selection operators are considered: binary tournament, rank selection and roulette wheel selection. The pseudocode of the EA is presented in Algorithm~\ref{alg:ea}.

\begin{algorithm}[!t]
\caption{Evolutionary algorithm (EA)}
\label{alg:ea}
\KwData{$f$, stop condition, popsize, selection, crossover and mutation operators (and parameters)}
\KwResult{best (found solution)}
$P \leftarrow$ generateRandomSolutions(popsize)\;
best $\leftarrow$ bestSolution($P$)\;
\While{not stopCondition()}{
	$x \leftarrow$ select($P$)\;
	$y \leftarrow$ select($P$)\;
	$z \leftarrow$ crossover($x$, $y$)\;
	$z \leftarrow$ mutation($z$)\;
	$w \leftarrow$ worstSolution($P$)\;
	\If{$f(z) > f(w)$}{
		$P \leftarrow P/\{w\} \cup \{z\}$\;
		\If{$f(z) > f(\text{best})$}{
			best $\leftarrow z$\;
		}
	}
}
\Return best
\end{algorithm}

We can also consider to apply DPX alone to solve the pseudo-Boolean optimization problems. 
That is, if we use two parent solutions, $x$ and $y$, that are complementary, $y_i=1-x_i$ for all $i=1,\ldots,n$, and we set $\beta=\infty$, DPX will return the global optimum. This approach should be useful in families of Mk Landscapes where the treewidth of the VIG is bounded by a small constant like, for example, the adjacent model of NK Landscapes, which can be solved in linear time~\citep{Wright2000}. In these cases, the cliques $C_i$ found in the clique tree can be fully explored due to the low number of variables they contain. However, we do not expect it to work well for NP-hard problems, like the random model of NKQ Landscapes that we use here, or the general MAX-SAT instances, that we also use in this section. 
The reason is that, in these cases, the cardinality of the cliques, $|C_i|$, usually grows with the number of variables, that is, $|C_i| \in \omega(1)$ and the complexity of DPX (with $\beta=\infty$) can be $2^{\omega(1)}$.
This approach was used by \citet{Ochoa2019gecco} to find global optimal solutions in randomly generated MAX-SAT instances of size $n=40$ variables. It was found that DPX required much time to exhaustively explore the search space and the authors decomposed the search in different parallel non-overlapping searches by fixing the values of some variables in the parent solutions before applying DPX. For the sake of completeness, however, we analyze the execution of DPX alone in Section~\ref{subsec:isodpx}.

For the experiments in this section, we use NKQ Landscapes with $K=2$ and $K=5$ and MAX-SAT instances (unweighted and weighted). With this diverse setting for the experiments (two metaheuristics belonging to different families and two pseudo-Boolean problems with two categories of instances each), we want to test DPX in different scenarios to identify when DPX is useful for pseudo-Boolean optimization and when it is not.

\subsubsection{Parameter settings}
\label{subsec:parameters}

We used an automatic tuning tool, irace~\citep{Lopez-Ibanez2016}, to help us find a good configuration for each algorithm in the different scenarios. The parameters to be tuned are $\beta$ in DPX, $\alpha$ in DRILS, and $p_m$, the selection operator, and the population size in EA. The ranges of values for each of them are presented in Table~\ref{tab:irace-tuned-params}. The version of irace used is 3.4.1 and the budget is \numprint{1000} runs of the tuned algorithm. The rest of parameters of irace were set to the default. We applied irace independently for each combination of algorithm, crossover operator and family of instances used. This way we want to compare the algorithms using a good configuration in each scenario.

\begin{table}[!ht]
\caption{Parameters tuned by irace and their domain.}
\label{tab:irace-tuned-params}
\centering
\begin{tabular}{@{}lrclrclr@{}}
\toprule
\multicolumn{2}{c}{DPX} & & \multicolumn{2}{c}{DRILS} & & \multicolumn{2}{c}{EA} \\
\cmidrule{1-2} \cmidrule{4-5} \cmidrule{7-8}
Parameter & Domain & & Parameter & Domain & & Parameter & Domain \\
\midrule
$\beta$ & [0-5] & & $\alpha$ & [0,0.5] & & $p_m$ & [0,0.5] \\
& & & & & & popsize & [10,100] \\
& & & & & & selection & (tournament, rank, roulette) \\
\bottomrule
\end{tabular}
\end{table}

The stop condition in the search algorithms was set to reach a time limit of 60 seconds (one minute) during the irace tuning and also in the experiments of the following sections. We use runtime as stop condition because the number of fitness function evaluations is not a good metric for the computational budget in our case, where gray box optimization operators are taking profit from the VIG and the subfunction evaluations. This can be easily observed comparing the huge difference in runtime of UX and DPX in Table~\ref{tab:crossover-comp-runtime-10K}.
The concrete value for the time limit (60 seconds) was chosen based on our previous experience with these algorithms including gray box optimization. In previous works and experiments, where we stopped the experiments after 300 seconds, we found that the results are rather stable after 60 seconds, with few exceptions. Using this short time also allows us to perform a larger set of experiments in the same time obtaining more insight on how the algorithms work.

At the end of the tuning phase, irace provides several configurations that are considered equivalent in performance (irace does not find statistically significant differences among them). We took in all the cases the first of those configurations and we show it in Tables~\ref{tab:nkq-irace} and~\ref{tab:maxsat-irace} for each combination of algorithm, crossover operator and set of instances. These are the parameters used in the experiments of Sections~\ref{subsec:nkq} and~\ref{subsec:maxsat}. We did not check the convergence speed of irace and we do not know how far the parameters in Tables~\ref{tab:nkq-irace} and~\ref{tab:maxsat-irace} are from the best possible configurations. Thus, we cannot get any insight from the parameters computed by irace. Our goal is just to compare all the algorithms and operators using good configurations, and avoid bias due to parameter setting.

\begin{table}[!ht]
\caption{Configuration proposed by irace during the tuning phase of the algorithms for NKQ instances with $n=\numprint{10000}$ variables and $K=2, 5$.}
\label{tab:nkq-irace}
\centering
\begin{tabular}{@{}rcrrlrcrrlr@{}}
\toprule
&  & \multicolumn{4}{c}{$K=2$} &  & \multicolumn{4}{c}{$K=5$} \\
\midrule
\multicolumn{1}{l}{DRILS} & & $\beta$ & \multicolumn{1}{c}{$\alpha$} & & & & $\beta$ & \multicolumn{1}{c}{$\alpha$} \\
\cmidrule{3-6} \cmidrule{8-11}
{\scriptsize DPX} & & 1 & 0.2219 &  & &  & 3 & 0.0462 \\
{\scriptsize APX} & &   & 0.1873 &  & &  &   & 0.0231 \\
{\scriptsize PX}  & &   & 0.1240 &  & &  &   & 0.0191 \\
{\scriptsize NX}  & &   & 0.0154 &  & &  &   & 0.0268 \\
{\scriptsize UX}  & &   & 0.0159 &  & &  &   & 0.0238 \\
\midrule
\multicolumn{1}{l}{EA} & & $\beta$ & \multicolumn{1}{c}{$p_m$} & selection & popsize & & $\beta$ & \multicolumn{1}{c}{$p_m$} & selection & popsize \\
\cmidrule{3-6} \cmidrule{8-11}
{\scriptsize DPX} & & 3 & 0.0044 & roulette & 61  & & 2 & 0.0080 & rank     & 15 \\
{\scriptsize APX} & &   & 0.0172 & roulette & 72  & &   & 0.0002 & roulette & 27 \\
{\scriptsize PX}  & &   & 0.0084 & rank     & 47  & &   & 0.0034 & rank     & 70 \\
{\scriptsize NX}  & &   & 0.0007 & roulette & 37  & &   & 0.0008 & rank     & 54 \\
{\scriptsize UX}  & &   & 0.0003 & rank     & 41  & &   & 0.0006 & roulette & 14 \\
\bottomrule
\end{tabular}
\end{table}


\begin{table}[!ht]
\caption{Configuration proposed by irace during the tuning phase of the algorithms for unweighted and weighted MAX-SAT instances.}
\label{tab:maxsat-irace}
\centering
\begin{tabular}{@{}rcrrlrcrrlr@{}}
\toprule
&  & \multicolumn{4}{c}{Unweighted} &  & \multicolumn{4}{c}{Weighted} \\
\midrule
\multicolumn{1}{l}{DRILS} & & $\beta$ & \multicolumn{1}{c}{$\alpha$} & & & & $\beta$ & \multicolumn{1}{c}{$\alpha$} \\
\cmidrule{3-6} \cmidrule{8-11}
{\scriptsize DPX} & & 4 & 0.0582 &  & &  & 2 & 0.1832 \\
{\scriptsize APX} & &   & 0.0941 &  & &  &   & 0.1870 \\
{\scriptsize PX}  & &   & 0.0482 &  & &  &   & 0.0996 \\
{\scriptsize NX}  & &   & 0.0299 &  & &  &   & 0.0241 \\
{\scriptsize UX}  & &   & 0.0571 &  & &  &   & 0.0214 \\
\midrule
\multicolumn{1}{l}{EA} & & $\beta$ & \multicolumn{1}{c}{$p_m$} & selection & popsize & & $\beta$ & \multicolumn{1}{c}{$p_m$} & selection & popsize \\
\cmidrule{3-6} \cmidrule{8-11}
{\scriptsize DPX} & & 5 & 0.0038 & rank       & 18  & & 2 & 0.0018 & rank       & 52 \\
{\scriptsize APX} & &   & 0.0096 & tournament & 19  & &   & 0.0069 & tournament & 24 \\
{\scriptsize PX}  & &   & 0.0051 & tournament & 27  & &   & 0.0086 & rank       & 20 \\
{\scriptsize NX}  & &   & 0.0047 & rank       & 18  & &   & 0.0020 & rank       & 27 \\
{\scriptsize UX}  & &   & 0.0019 & rank       & 18  & &   & 0.0012 & tournament & 78 \\
\bottomrule
\end{tabular}
\end{table}


In order to reduce the bias due to the stochastic nature of the algorithms, they were run ten times for each instance and average results are presented in the next sections. The Mann-Whitney test was run with the samples from the ten independent runs to check if the observed difference in the performance of the algorithms for each instance is statistically significant at the $0.05$ confidence level or not.

\subsubsection{Results for NKQ Landscapes}
\label{subsec:nkq}

In this section we analyze the results obtained for NKQ Landscapes. We do not know the fitness value of the global optimal solutions in the instances. For this reason, for each instance (there are ten instances per value of $K$) we computed the fitness of the best solution found in any run by any algorithm+crossover combination, $f^*$. We used $f^*$ to  normalize the fitness of the best solutions found by the algorithms in the instance. Thus, if $x$ is the best solution found by an algorithm in an instance, we define the \emph{quality} of solution $x$ as $q(x)=f(x)/f^*$. The benefit of this is that the quality of a solution $x$ is a real number between 0 and 1 that measures how near is the fitness of $x$ to the best known fitness (in this set of experiments). This also allows us to aggregate quality values from different instances because they are normalized to the same range. Higher values of quality are better. 

The main results of the section are shown in Table~\ref{tab:nkq-comparison}, where the column \emph{quality} is the average over ten runs and ten instances (100 samples in total) of the quality of the best solution found by the algorithm+crossover combination in the row. The column \emph{statistical difference} shows the result of a Mann-Whitney test (with significance level 0.05) and a median comparison to check if the differences observed in quality between the algorithm+crossover in the row and algorithm+DPX are statistically significant or not.
In each row and for each instance, the results of the ten runs are used as input to the Mann-Whitney test. It determines equivalence or dissimilarity of samples. The sense of the inequality is determined by the median comparison.
In the table, the numbers followed by a black triangle ($\blacktriangle$), white triangle ($\triangledown$) and equal sign ($=$) are the numbers of instances in which the algorithm+crossover combination in the row is statistically better, worse or similar to algorithm+DPX. In Figures~\ref{fig:drils-quality-nk} and~\ref{fig:ea-quality-nk} we show the average quality of the best found solution at any time during the search using the different algorithms and crossover operators. We group the  curves in the figures by algorithm (DRILS and EA) and ruggedness of the instances ($K=2$ and $K=5$).

\begin{table}[!ht]
\caption{Performance of the five recombination operators used in DRILS and EA when solving NKQ Landscapes instances with $n=\numprint{10000}$ variables.
The symbols $\blacktriangle$, $\triangledown$ and $=$ are used to indicate that the use of the crossover operator in the row yields statistically better, worse or similar results than the use of DPX in each algorithm.
}
\label{tab:nkq-comparison}
\centering
\begin{tabular}{@{}rcccrcccr@{}}
\toprule
&  & \multicolumn{3}{c}{$K=2$} &  & \multicolumn{3}{c}{$K=5$} \\
\cmidrule{3-5} \cmidrule{7-9}
& & \multicolumn{1}{c}{Statistical difference} & & Quality & & \multicolumn{1}{c}{Statistical difference}  & & Quality \\
\midrule
\multicolumn{1}{l}{DRILS} \\
{\scriptsize DPX} & &  & & 0.9997 & &  & & 0.9972 \\
{\scriptsize APX} & & $0\blacktriangle\ \phantom{0}8\triangledown\ 2=$ & & 0.9995 & & $\phantom{0}0\blacktriangle\ \phantom{0}7\triangledown\ 3=$ & & 0.9947 \\
{\scriptsize PX} & & $0\blacktriangle\ 10\triangledown\ 0=$ & & 0.9990 & & $\phantom{0}0\blacktriangle\ \phantom{0}7\triangledown\ 3=$ & & 0.9949 \\
{\scriptsize NX} & & $0\blacktriangle\ 10\triangledown\ 0=$ & & 0.9786 & & $\phantom{0}0\blacktriangle\ 10\triangledown\ 0=$ & & 0.9934 \\
{\scriptsize UX} & & $0\blacktriangle\ 10\triangledown\ 0=$ & & 0.9790 & & $\phantom{0}0\blacktriangle\ 10\triangledown\ 0=$ & & 0.9935 \\
\midrule
\multicolumn{1}{l}{EA} \\
{\scriptsize DPX} & &  & & 0.9795 & &  & & 0.8132 \\
{\scriptsize APX} & & $0\blacktriangle\ 10\triangledown\ 0=$ & & 0.9568 & & $\phantom{0}1\blacktriangle\ \phantom{0}0\triangledown\ 9=$ & & 0.8890 \\
{\scriptsize PX} & & $0\blacktriangle\ 10\triangledown\ 0=$ & & 0.9445 & & $10\blacktriangle\ \phantom{0}0\triangledown\ 0=$ & & 0.9085 \\
{\scriptsize NX} & & $0\blacktriangle\ 10\triangledown\ 0=$ & & 0.8803 & & $\phantom{0}0\blacktriangle\ \phantom{0}1\triangledown\ 9=$ & & 0.7811 \\
{\scriptsize UX} & & $0\blacktriangle\ 10\triangledown\ 0=$ & & 0.9313 & & $\phantom{0}0\blacktriangle\ \phantom{0}1\triangledown\ 9=$ & & 0.8407 \\
\bottomrule
\end{tabular}
\end{table}





The first important conclusion we obtain from the results in Table~\ref{tab:nkq-comparison} is that DRILS performs better with DPX than with any other crossover operator. There is no single NKQ Landscapes instance in our experimental setting where another crossover operator outperforms DPX included in DRILS. There are only a few instances (8 in total) where APX and/or PX show a similar performance. We can observe in Figure~\ref{fig:drils-quality-nk} (a) that DRILS+DPX obtains the best average quality at any time during the search when $K=2$, followed by PX and APX. UX and NX provide the worst average quality in this set of instances. We observe in the figure signs of convergence in all the crossover operators. However, after a careful analysis checking the time of last improvement, whose distribution is presented in Figure~\ref{fig:tli-vioplot-k2} (a), we notice that DRILS with DPX, APX and PX provides improvements to the best solution after 50 seconds in around 50\% of the runs, while DRILS with UX and NX seems to stuck in 30 to 40 seconds after the start of the search, and much earlier in some cases. We wonder if this time could be biased by the different runtime of crossover operators. Perhaps the algorithm produces the last improvement near the end of the execution for DPX, APX and PX but the previous one was in the middle of the run, earlier than the last improvement of NX and UX. To investigate this, we show in 
Figure~\ref{fig:tli-vioplot-k2} (b) the distribution of the average time between improvements for the last three improvements. This time is far below one second in most of the cases for all the crossover operators, which means that they are producing better solutions several times per second on average until the time of last improvement, and there is no bias related to the different crossover runtime.

\begin{figure}[!ht]
\centering
\subfloat[][$K=2$]{
	\includegraphics[width=0.45\textwidth]{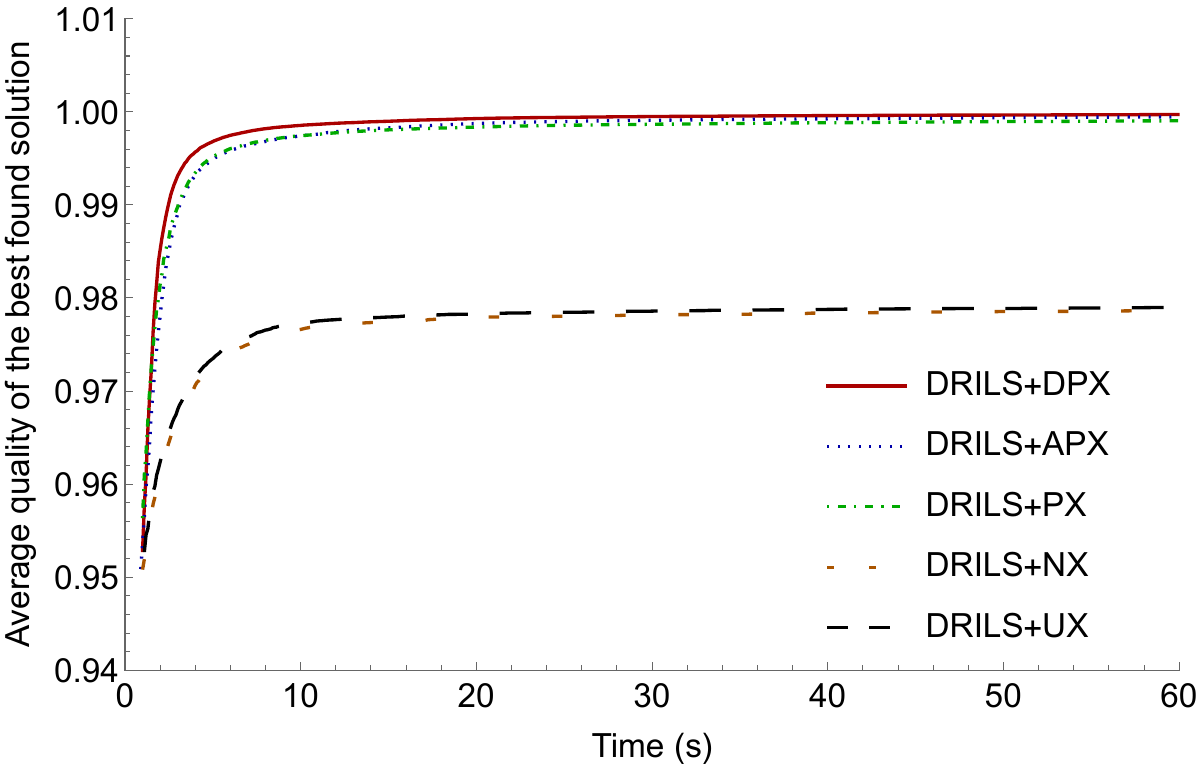}
}
\qquad
\subfloat[][$K=5$]{
	\includegraphics[width=0.45\textwidth]{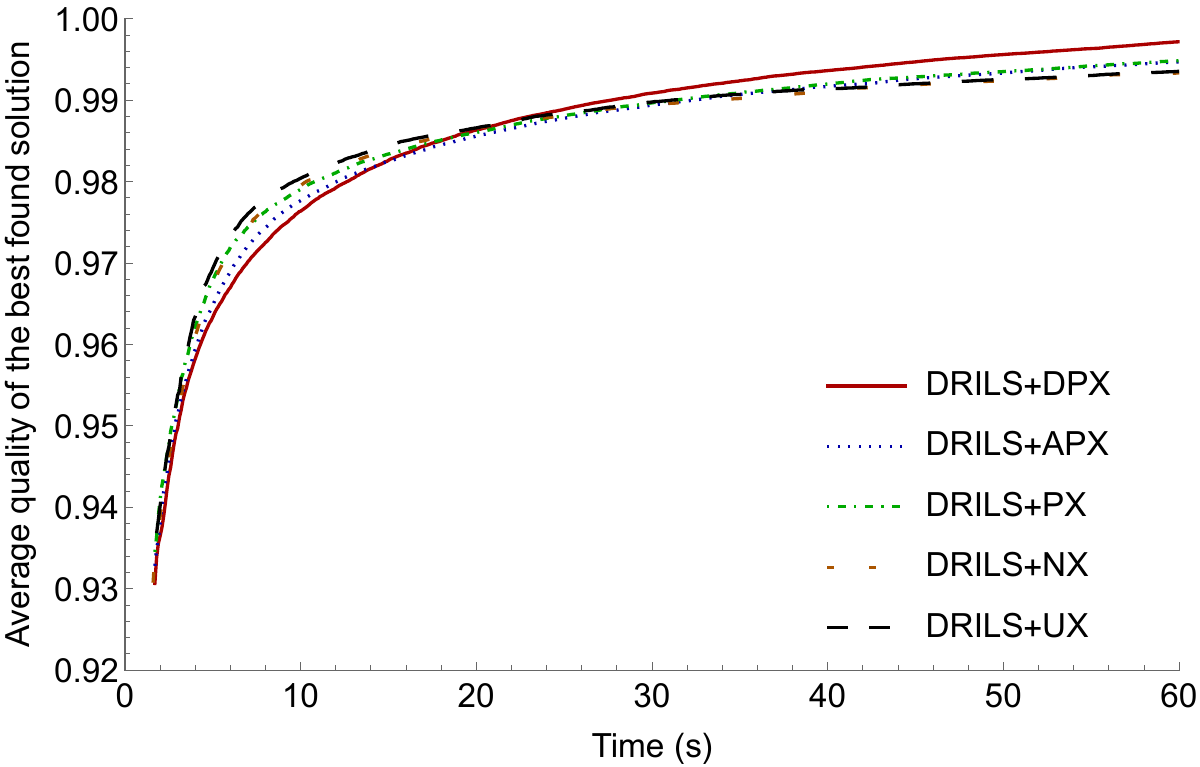}
}
\caption{Average quality of the best found solution at any time for DRILS using different crossover operators solving NKQ Landscapes.}
\label{fig:drils-quality-nk}
\end{figure}

%

\begin{figure}[!ht]
\centering
\subfloat[][Time of last improvement (in seconds)]{
	\includegraphics[width=0.45\textwidth]{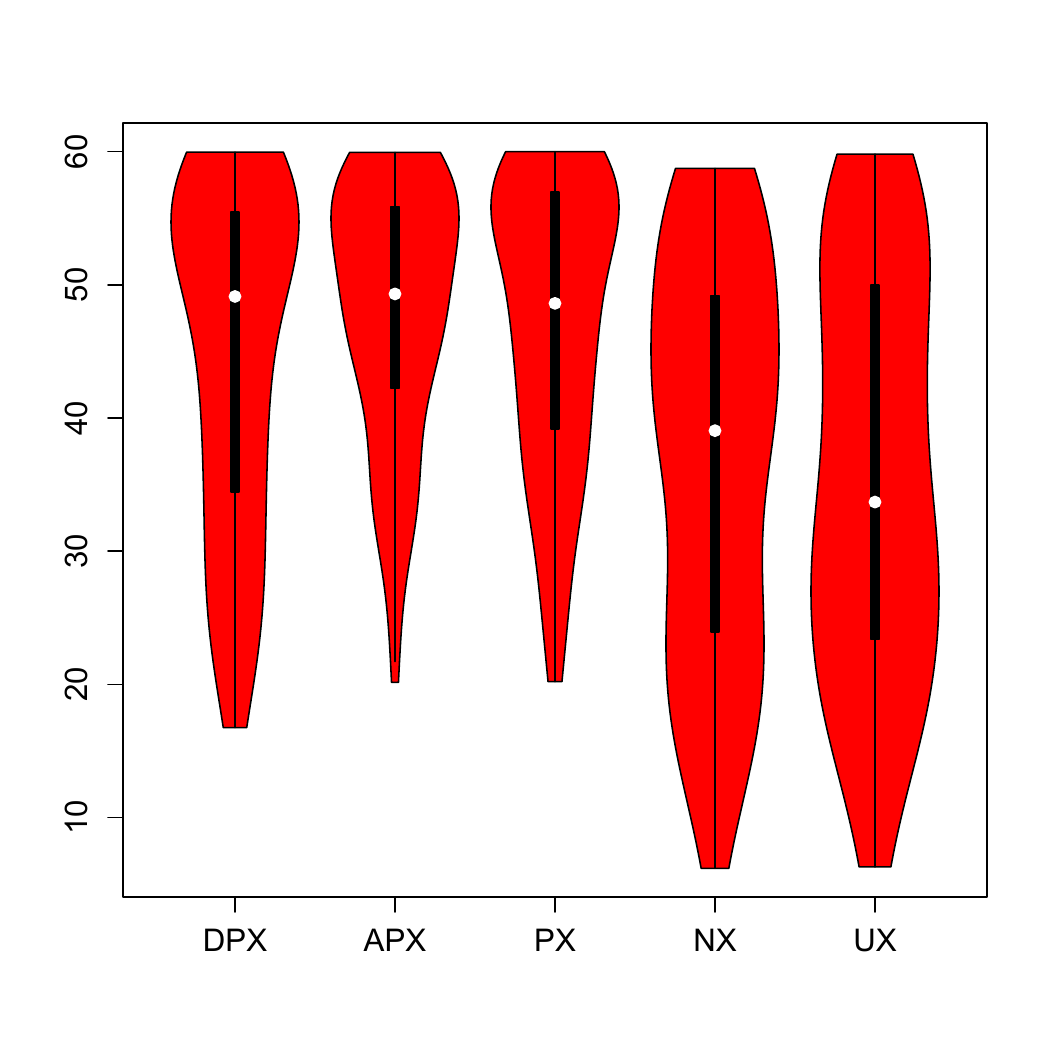}
}
\qquad
\subfloat[][Average time between the three last improvements (in seconds)]{
	\includegraphics[width=0.45\textwidth]{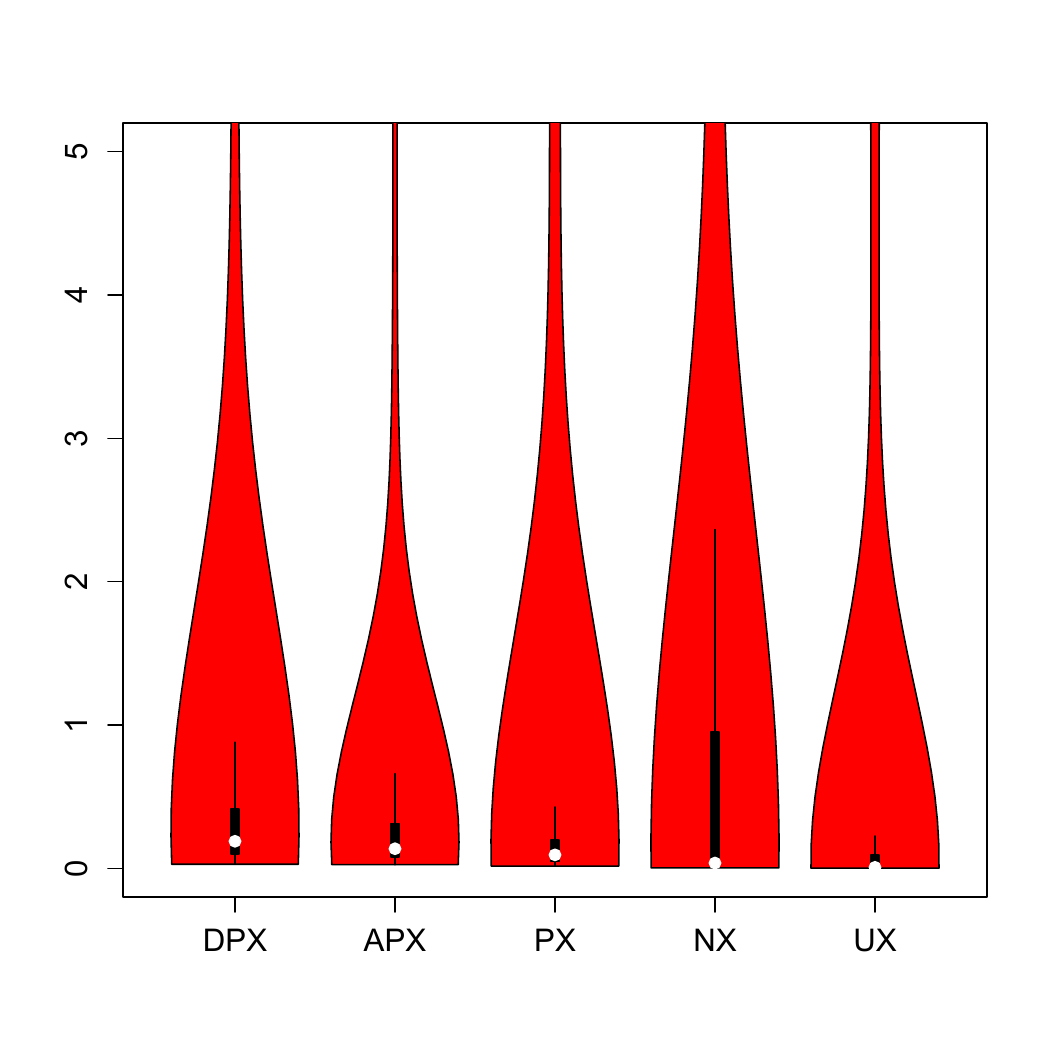}
}
\caption{Probability distribution of the time of last improvement to the best found solution from the start of the search (a) and the average time between improvements for the three last improvements (b) in DRILS with the different crossover operators when solving NKQ Landscapes with $K=2$.}
\label{fig:tli-vioplot-k2}
\end{figure}

In the more rugged instances ($K=5$), shown in Figure~\ref{fig:drils-quality-nk} (b), DRILS+DPX is the best performing algorithm after 20 seconds of computation. Before that time DRILS+UX provides the best performance. We can explain this with the help of Table~\ref{tab:crossover-comp-runtime-10K}. UX is the fastest crossover operator, and helps to advance in the search at the beginning. DPX (as well as PX and APX) are slower operators and, even if they provide better quality offspring, they slow down the search, requiring more time to be effective. We have to recall here that DRILS includes a hill climber, what explains why using a random black box operator like UX the quality of the best solution is still high (above 0.93).

If we analyze the performance of the crossover operators in EA, we observe that DPX is also the best crossover operator for the instances with $K=2$. However, when $K=5$ DPX is outperformed by PX and, in general, shows a performance similar to APX, NX and UX. The best crossover operator in EA when $K=5$ is PX. Taking a look to Figure~\ref{fig:ea-quality-nk} (a) we observe that EA+PX and EA+APX provide the highest average quality of the best solution during the first 35 to 40 seconds, then they are surpassed by DPX. 
The reason for this slow behaviour of EA+DPX is again the high runtime of DPX, which in the case of EA is higher than in the case of DRILS due to the fact that the solutions in the initial population  are random and, thus, differ in around $h= 0.5 n = \numprint{5000}$ bits.
This slow runtime is specially critical when $K=5$ (Figure~\ref{fig:ea-quality-nk} (b)), where EA+DPX is not able to reach the average quality of EA+PX in 60 seconds. 
In Figure~\ref{fig:ea-drils-runtime-nk-k5} we plot the time required by DPX during the search when it is included in DRILS and EA. The runtime of DPX in DRILS is a few milliseconds because the parent solutions differ in around $\alpha n = 462$ bits ($\alpha=0.0462$ according to Table~\ref{tab:nkq-irace}), while the runtime of DPX in EA starts in six seconds and goes down to one second at the end of the search. This behaviour suggests that in an EA a hybrid approach combining PX at the beginning of the search and DPX later during the search could be a better strategy to reach better quality solutions in a short time.
The random crossover operators, UX and NX, show a poorer performance in EA compared to DRILS probably because there is no local search in EA.

\begin{figure}[!ht]
\centering
\subfloat[][$K=2$]{
	\includegraphics[width=0.45\textwidth]{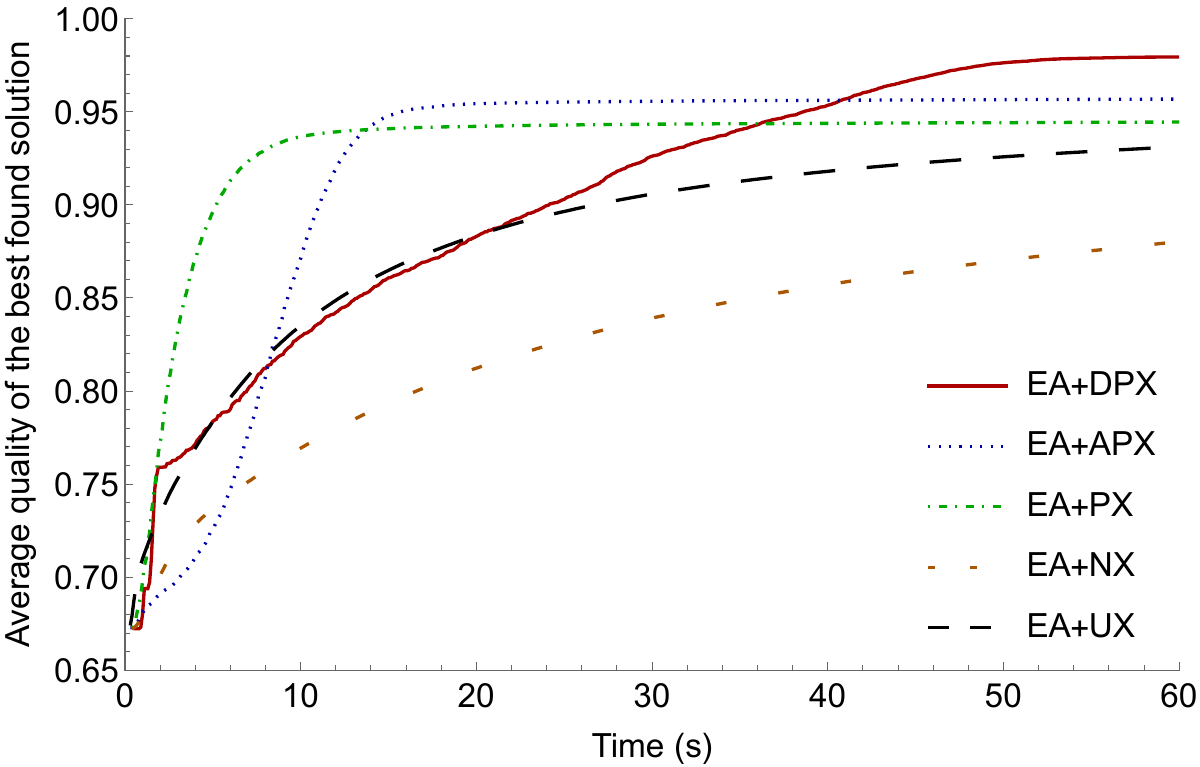}
}
\qquad
\subfloat[][$K=5$]{
	\includegraphics[width=0.45\textwidth]{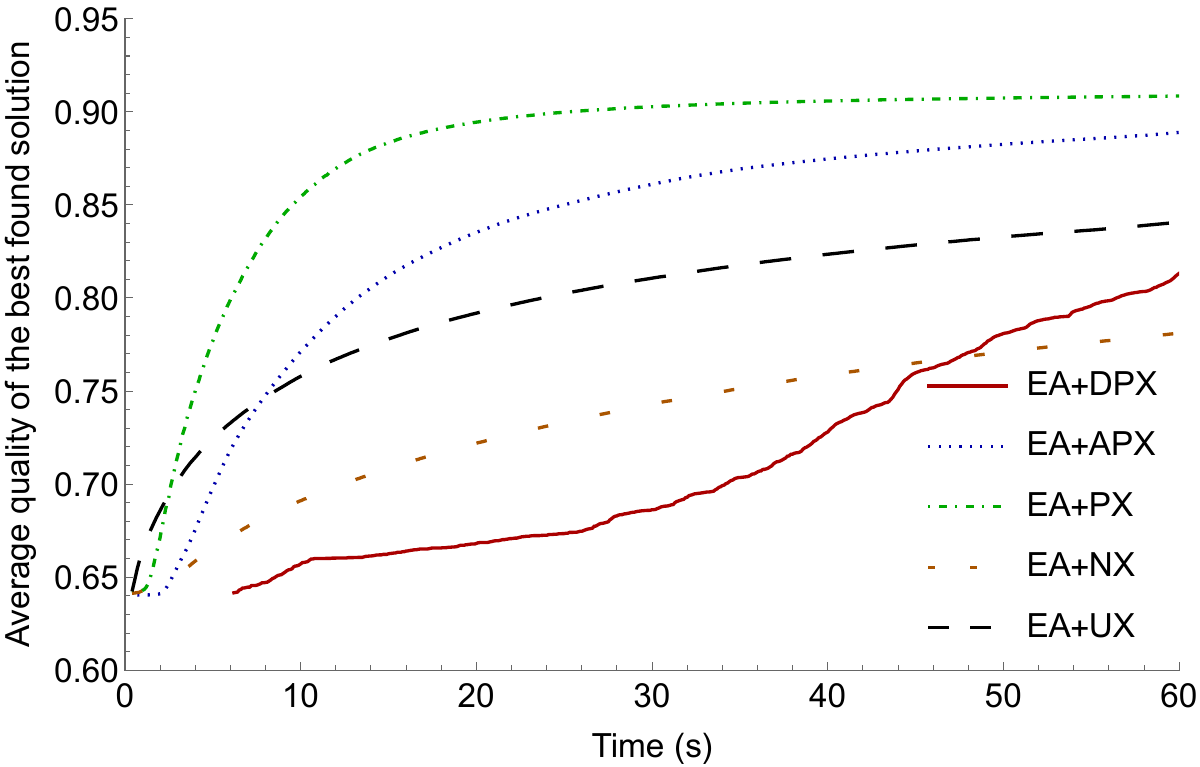}
}
\caption{Average quality of the best found solution at any time for EA using different crossover operators solving NKQ Landscapes.}
\label{fig:ea-quality-nk}
\end{figure}

%

\begin{figure}[!ht]
\centering
\includegraphics[width=0.5\textwidth]{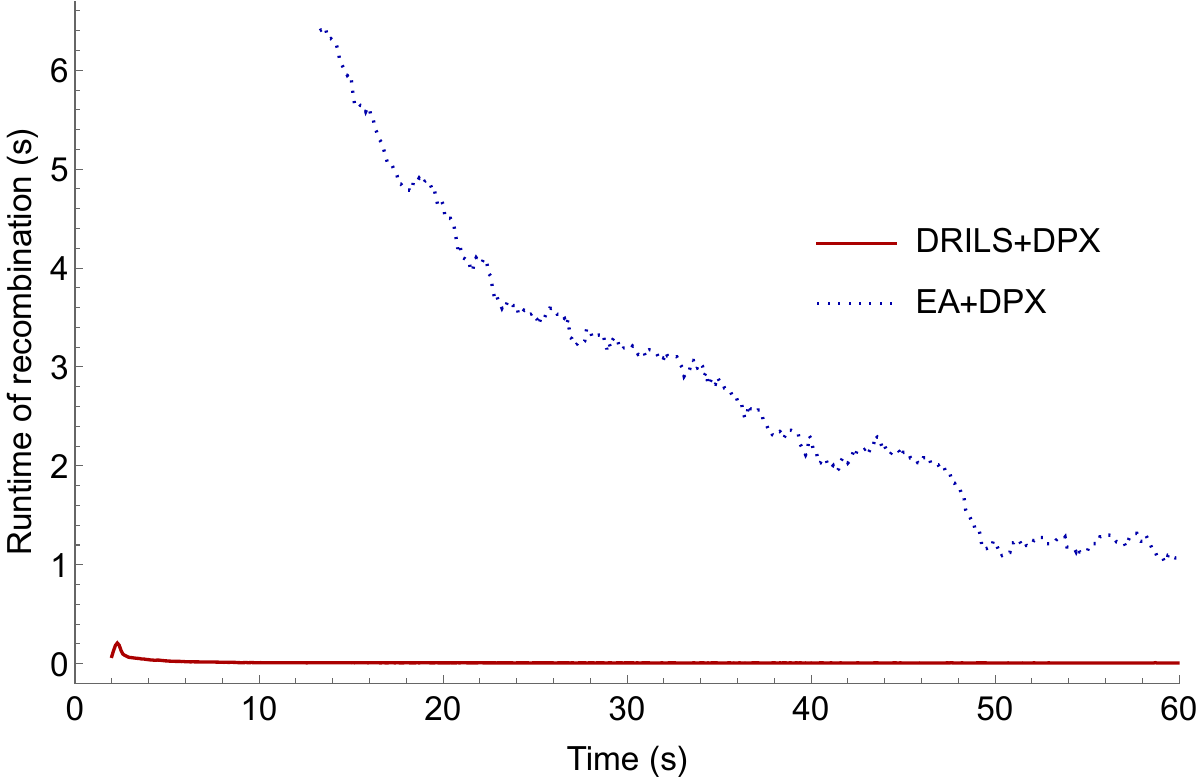}
\caption{Runtime (in seconds) of DPX as the search progresses in DRILS and EA solving NKQ Landscapes with $K=5$.}
\label{fig:ea-drils-runtime-nk-k5}
\end{figure}

Finally, although it is not our goal to compare search algorithms (only crossover operators) we would like to highlight some observations regarding the average quality of the best found solutions in DRILS and EA. We conclude that DRILS is always better than EA. For $K=2$, the highest quality in EA is obtained when DPX is used and its quality is only slightly higher than that of DRILS with NX and UX (the worst performing crossover operators in DRILS). For $K=5$ the difference is even higher: EA+PX reaches an average quality of 0.9085 (highest for EA) which is far below any average quality of DRILS, all above 0.9934 (the one of DRILS+NX). We think that the reason for that could be the presence of a local search operator in DRILS, while EA is mainly guided by selection and crossover (when PX, APX or DPX is used).

\subsubsection{Results for MAX-SAT}
\label{subsec:maxsat}

In this section we analyze the results obtained for MAX-SAT. We also use the quality of the solutions defined in Section~\ref{subsec:nkq} as a normalized measure of quality.
Table~\ref{tab:maxsat-comparison} presents the main results of the section. The meaning of the columns is the same as in Table~\ref{tab:nkq-comparison}. 
In the case of DRILS all the instances are used in the statistical tests and the computation of the average quality (160 instances in the unweighted category and 132 instances in the weighted category).
In the case of EA, we observed that it failed to complete the execution in some runs for some instances when it was combined with DPX. The reason was an out of memory problem, caused by the large number of differing variables among the solutions in the initial generations. In this case, we only computed the average quality for instances in which at least 90\% of the runs were successful (nine of the ten runs) and we manually counted the instances with less than 90\% successful runs as significantly worse than EA+DPX for the remaining EA+crossover combinations in Table~\ref{tab:maxsat-comparison} without performing any statistical test. Nine unweighted instances and five weighted instances had less than 90\% successful EA+DPX runs. Three unweighted instances and no weighted instance had exactly 90\% of successful runs and in the remaining instances EA+DPX ends successfully in all the runs.

\begin{table}[!ht]
\caption{Performance of the five recombination operators used in DRILS and EA when solving MAX-SAT instances.
The symbols $\blacktriangle$, $\triangledown$ and $=$ are used to indicate that the use of the crossover operator in the row yields statistically better, worse or similar results than the use of DPX.
}
\label{tab:maxsat-comparison}
\centering
\begin{tabular}{@{}rcccrcccr@{}}
\toprule
&  & \multicolumn{3}{c}{Unweighted} &  & \multicolumn{3}{c}{Weighted} \\
\cmidrule{3-5} \cmidrule{7-9}
& & \multicolumn{1}{c}{Statistical difference} & & Quality & & \multicolumn{1}{c}{Statistical difference}  & & Quality \\
\midrule
\multicolumn{1}{l}{DRILS} \\
{\scriptsize DPX} & &  & & 0.9984 & &  & & 0.9996 \\
{\scriptsize APX} & & $14\blacktriangle\ \phantom{0}91\triangledown\ 55=$ & & 0.9973 & & $15\blacktriangle\ \phantom{0}86\triangledown\ 31=$ & & 0.9984 \\
{\scriptsize PX} & & $\phantom{0}8\blacktriangle\ 103\triangledown\ 49=$ & & 0.9968 & & $25\blacktriangle\ \phantom{0}80\triangledown\ 27=$ & & 0.9982 \\
{\scriptsize NX} & & $\phantom{0}2\blacktriangle\ 126\triangledown\ 32=$ & & 0.9946 & & $\phantom{0}1\blacktriangle\ 126\triangledown\ \phantom{0}5=$ & & 0.9915 \\
{\scriptsize UX} & & $\phantom{0}0\blacktriangle\ 124\triangledown\ 36=$ & & 0.9953 & & $\phantom{0}1\blacktriangle\ 126\triangledown\ \phantom{0}5=$ & & 0.9930 \\
\midrule
\multicolumn{1}{l}{EA} \\
{\scriptsize DPX} & &  & & 0.9644 & &  & & 0.9583 \\
{\scriptsize APX} & & $52\blacktriangle\ \phantom{0}68\triangledown\ 40=$ & & 0.9604 & & $43\blacktriangle\ \phantom{0}63\triangledown\ 26=$ & & 0.9649 \\
{\scriptsize PX} & & $17\blacktriangle\ 107\triangledown\ 36=$ & & 0.9095 & & $\phantom{0}8\blacktriangle\ 109\triangledown\ 15=$ & & 0.9057 \\
{\scriptsize NX} & & $18\blacktriangle\ 101\triangledown\ 41=$ & & 0.8980 & & $18\blacktriangle\ 103\triangledown\ 11=$ & & 0.8786 \\
{\scriptsize UX} & & $27\blacktriangle\ \phantom{0}96\triangledown\ 37=$ & & 0.9134 & & $18\blacktriangle\ \phantom{0}99\triangledown\ 15=$ & & 0.8989 \\
\bottomrule
\end{tabular}
\end{table}





From the results in Table~\ref{tab:maxsat-comparison} we conclude that both algorithms (DRILS and EA) perform better, in general, using DPX as the crossover operator. Only in very few cases any other crossover operator improves the final result of DRILS. 
When EA is used, the difference is not so clear, but still significant.

Once again, we also observe that the performance of DRILS is better than that of EA. The maximum average quality in EA is 0.9649 (EA+APX in the weighted instances) while the average quality of DRILS is always above 0.9915 for all the crossover operators and category of instances.

We do not expect DRILS or EA to be competitive with state-of-the-art incomplete MAX-SAT solvers like SATLike-c\footnote{SATLike-c was the winner in the unweighted incomplete track of the MAX-SAT Evaluation 2021 and got third position in the weighted incomplete track.}~\citep{CAI2020103354}, because they are general optimization algorithms. However, DPX could be useful to improve the performance of some incomplete MAX-SAT solvers, as PX did in recent work~\citep{Chen2018gecco}.

\subsubsection{DPX alone as search algorithm}
\label{subsec:isodpx}

In this section we analyze the results of DPX alone used to solve pseudo-Boolean optimization problems. We apply DPX to a random solution and its complement, with the goal of finding the global optimum. Due to technical limitations of our current implementation of DPX, we cannot set $\beta=\infty$, but we use the maximum value of $\beta$ allowed by the implementation, which is 28. This also means that if the whole search space is not explored for a concrete instance the result could depend on the initial (random) solution. For this reason, we run DPX ten times per instance on different random solutions. We set a runtime limit of 12 hours.

After applying DPX to the 20 instances of NKQ Landscapes (ten instances for each value of $K$), we found that all the runs were prematurely terminated due to an out of memory error (the memory limit was set to 5GB as in the previous experiments).
In the case of the MAX-SAT instances, the runs finished in less than 12 hours for eight unweighted instances and two  weighted instances. In 152 unweighted instances and 128 weighted ones, DPX ran out of memory. In two weighted instances, we stopped DPX after 12 hours of computation. Table~\ref{tab:dpx28-maxsat} shows the MAX-SAT instances that finished without error and, for each one, it shows the average and maximum number of satisfied clauses in the ten independent runs of DPX, the average runtime (in seconds) and the logarithm in base 2 of the implicitly explored solutions (it is the same in all runs). The minimum number of satisfied clauses found in all the runs of DRILS+DPX in less than 60 seconds is shown for comparison in the last column. We mark with an asterisk in the logarithm of implicitly explored solutions the runs that fully explored the search space (thus, DPX was able to certify the global optimum).


\begin{table}[!ht]
\caption{Results of DPX with $\beta=28$ for the MAX-SAT instances that finished in less than 12 hours without memory error. The minimum number of satisfied clauses found in all the runs of DRILS+DPX in less than 60 seconds is shown for comparison in the last column. The asterisk means that the whole search space was explored and the number of satisfied clauses is the global maximum.}
\label{tab:dpx28-maxsat}
\centering
\begin{tabular}{@{}lcrcrcrcrcr@{}}
\toprule
Instance & & \multicolumn{3}{c}{Satisfied clauses} & & \multicolumn{1}{c}{Time} & & Explored & & DRILS \\
\cmidrule{3-5}
         & & \multicolumn{1}{c}{avg} & & \multicolumn{1}{c}{max}       & &    \multicolumn{1}{c}{(s)}      & &   \multicolumn{1}{c}{($\log_2$)}           & &  +DPX         \\
\midrule
\multicolumn{11}{c}{Unweighted instances} \\
bcp-msp-normalized-ii16c1 & &  \numprint{15690} & & \numprint{15848} & & \numprint{9401} & & 49 & & \numprint{20266} \\
maxcut-brock200\_1.clq & & 889 & & 892 & & \numprint{8253} & & 32 & & 894 \\
maxcut-brock400\_2.clq & & 929 & & 933 & & \numprint{6387} & & 32 & & 936 \\
maxcut-brock800\_2.clq & & 809 & & 816 & & \numprint{8572} & & 31 & & 819 \\
maxcut-p\_hat1000-2.clq & & 600 & & 600 & & 375 & & *40 & & 600 \\
maxcut-p\_hat700-3.clq & & 986 & & 989 & & \numprint{7617} & & 32 & & 993 \\
maxcut-san400\_0.5\_1.clq & & 644 & & 644 & & 990 & & *40 & & 644 \\
maxcut-san400\_0.9\_1.clq & & 1047 & & 1050 & & \numprint{8589} & & 30 & & 1052 \\
\midrule
\multicolumn{11}{c}{Weighted instances} \\
maxcut-johnson8-2-4.clq & &  \numprint{2048} & & \numprint{2048} & & 144 & & *28 & & \numprint{2048} \\
maxcut-sanr200\_0.9.clq & & \numprint{5994} & & \numprint{6027} & & \numprint{10680} & & 31 & & \numprint{6036} \\
\bottomrule
\end{tabular}
\end{table}

The main conclusion is that DPX alone is not good compared to its combination with a search algorithm. It requires a lot of memory and time to compute the result. Even when DPX finishes, the results are always outperformed by DRILS+DPX (see last column of Table~\ref{tab:dpx28-maxsat}).

In the previous experiment, we forced DPX to use a lot of memory and runtime because the value of $\beta$ was high. One of the advantages of DPX compared to an optimal recombination operator is that we can limit the time and space complexity of DPX using $\beta$. We wonder if DPX alone with a low value of $\beta$ could outperform DRILS or EA. In order to answer this question we designed a new experiment in which we run a new algorithm, called iDPX (iterated DPX) which consists in the iterated application of DPX over a random solution and its complement. The algorithm stops when a predefined time limit is reached (one minute in our case). We used irace to tune iDPX. The only parameter to tune is $\beta$, and we used the same configuration of irace that we used to tune DRILS and EA. According to irace, the best value for $\beta$ is 5 for both, NKQ Landscapes and MAX-SAT.
We run iDPX ten times per instance and compared it with DRILS and EA using the five crossover operators. 

In NKQ Landscapes, iDPX was statistically significantly worse than all the remaining ten algorithms (DRILS and EA with the five crossover operators) for all the instances. The average quality of the solutions obtained by iDPX was 0.7543 for $K=2$ and 0.6640 for $K=5$, which are very low values compared to the ones obtained by the other algorithms (see Table~\ref{tab:nkq-comparison}).

In MAX-SAT, iDPX prematurely finished with out of memory errors in 59 unweighted instances and 25 weighted ones. In these cases, we ran iDPX with $\beta=0$ to alleviate the memory requirements but the algorithm ran out of memory again. When iDPX does not have a result for an instance due to an error, we mark it as statistically significant worse than the other algorithms that finished with a result and equal to the algorithms that did not finished (EA+DPX does not finish in some instances). In Table~\ref{tab:maxsat-isodpx} we compare iDPX with the other ten algorithms, providing the same metrics as in Section~\ref{subsec:maxsat}. The average quality values change with respect to the ones in Table~\ref{tab:maxsat-comparison}. The reason is that the instances used for the average computation in Table~\ref{tab:maxsat-isodpx} is a subset of the instances used in Table~\ref{tab:maxsat-comparison}, due to the memory errors of iDPX.

\begin{table}[!ht]
\caption{Performance comparison of iDPX, DRILS and EA when solving MAX-SAT instances.
The symbols $\blacktriangle$, $\triangledown$ and $=$ are used to indicate that the algorithm in the row yields statistically better, worse or similar results than the use of iDPX.
}
\label{tab:maxsat-isodpx}
\centering
\begin{tabular}{@{}rcccrcccr@{}}
\toprule
&  & \multicolumn{3}{c}{Unweighted} &  & \multicolumn{3}{c}{Weighted} \\
\cmidrule{3-5} \cmidrule{7-9}
& & \multicolumn{1}{c}{Statistical difference} & & Quality & & \multicolumn{1}{c}{Statistical difference}  & & Quality \\
\midrule
\multicolumn{1}{l}{iDPX} & &  & & 0.8772 & &  & & 0.8516 \\
\midrule
\multicolumn{1}{l}{DRILS} \\
{\scriptsize DPX} & & $160\blacktriangle\ \phantom{0}0\triangledown\ \phantom{0}0=$ & & 0.9992 & & $132\blacktriangle\ \phantom{0}0\triangledown\ \phantom{0}0=$ & & 0.9997 \\
{\scriptsize APX} & & $160\blacktriangle\ \phantom{0}0\triangledown\ \phantom{0}0=$ & & 0.9979 & & $132\blacktriangle\ \phantom{0}0\triangledown\ \phantom{0}0=$ & & 0.9984 \\
{\scriptsize PX}  & & $160\blacktriangle\ \phantom{0}0\triangledown\ \phantom{0}0=$ & & 0.9981 & & $132\blacktriangle\ \phantom{0}0\triangledown\ \phantom{0}0=$ & & 0.9982 \\
{\scriptsize NX}  & & $160\blacktriangle\ \phantom{0}0\triangledown\ \phantom{0}0=$ & & 0.9966 & & $130\blacktriangle\ \phantom{0}1\triangledown\ \phantom{0}1=$ & & 0.9907 \\
{\scriptsize UX}  & & $160\blacktriangle\ \phantom{0}0\triangledown\ \phantom{0}0=$ & & 0.9967 & & $130\blacktriangle\ \phantom{0}1\triangledown\ \phantom{0}1=$ & & 0.9925 \\
\multicolumn{1}{l}{EA} \\
{\scriptsize DPX} & & $150\blacktriangle\ \phantom{0}0\triangledown\ 10=$           & & 0.9869 & & $126\blacktriangle\ \phantom{0}0\triangledown\ \phantom{0}6=$ & & 0.9700 \\
{\scriptsize APX} & & $148\blacktriangle\ \phantom{0}1\triangledown\ 11=$           & & 0.9675 & & $126\blacktriangle\ \phantom{0}0\triangledown\ \phantom{0}6=$ & & 0.9676 \\
{\scriptsize PX}  & & $131\blacktriangle\ 16\triangledown\ 13=$                     & & 0.9344 & & $112\blacktriangle\ \phantom{0}3\triangledown\ 17=$           & & 0.9155 \\
{\scriptsize NX}  & & $128\blacktriangle\ 20\triangledown\ 12=$                     & & 0.9156 & & $\phantom{0}95\blacktriangle\ 27\triangledown\ 10=$           & & 0.8836 \\
{\scriptsize UX}  & & $131\blacktriangle\ 19\triangledown\ 10=$                     & & 0.9334 & & $120\blacktriangle\ \phantom{0}4\triangledown\ \phantom{0}8=$ & & 0.9063 \\
\bottomrule
\end{tabular}
\end{table}

We observe in Table~\ref{tab:maxsat-isodpx} that iDPX is statistically worse than DRILS+DPX, DRILS+APX and DRILS+PX in all the instances. iDPX does not outperform EA+DPX in any case (it is equivalent in 16 instances). Furthermore, iDPX is worse in most of the instances for algorithms not using DPX (e.g., EA+UX). Thus, our main conclusion of this section is that DPX alone is not competitive with any search algorithm and should only be used as a crossover operator inside a search algorithm. DPX can provide the global optimum only in the cases where the number of variables is low or the VIG has a low number of edges. This is hardly the case in NP-hard problems (it only happened in three MAX-SAT instances of our benchmark).

\subsection{Local optima network analysis of DRILS+DPX}
\label{subsec:lon}

The best algorithm+crossover combination identified in Section~\ref{subsec:algorithms-comp} is DRILS+DPX, for NKQ Landscapes and MAX-SAT. In this section we conduct a local optima network (LON) analysis~\citep{lon-gecco08,Ochoa2015gecco} in order to better understand the search dynamics of DRILS+DPX. 
A LON is a graph where nodes are local optima and edges represent search transitions among them with a given operator. DRILS has two transition operators, crossover (DPX) and perturbation (each followed by local search to produce a local optimum), which are modelled as two types of edges in the LON: DPX + hill climber and perturbation + hill climber. 
An edge is, respectively, \emph{improving} if its end node has better fitness than its start node, \emph{equal} if the start and end nodes have the same fitness, and \emph{worsening}, if the end node has worse fitness than the start node.  
Our analysis reports the LONs extracted by ten runs of DRILS+DPX solving NKQ Landscapes instances with $n=\numprint{10000}$ and the two extreme values of $K$ used in our experiments (2 and 5). The parameters used are the ones in Table~\ref{tab:irace-tuned-params} for this algorithm and value of $K$.
The process of extracting the LON data is as follows. For each of the ten runs per instance, we record every unique local optima and edge encountered, from the start until the end of the run. As it is done in previous work \citep{OchoaV-joh18}, we combined the search trajectories produced by the ten runs as a sampling process to construct a single LON and our analysis revealed that there are no overlapping nodes and edges across the runs. Therefore, the combined LONs contain ten connected components, one characterizing each run. 
Table~\ref{tab:lon-stats} reports some basic LON statistics, specifically, the number of nodes (local optima) and the number edges of each type, crossover (DPX) and perturbation, grouped as improving, equal and worsening edges. They are statistics over the ten different instances (and LONs). There cannot be a worsening crossover edge, by design of DPX, but we include the column for the sake of completeness. 
The number of nodes and edges sampled is similar for the different values of $K$. Although DPX is slower for $K=5$, its perturbation factor $\alpha=0.0462$ is lower than in the case of $K=2$. 
With the parameters provided by irace, DRILS+DPX visits a similar number of unique local optima during the search for both values of $K$.
It is interesting to note that most of the crossover edges are improving, while most of the perturbation edges are deteriorating. This suggests that the role of perturbation within DRILS is to provide diversity as a raw material for DPX, which then incorporates newly found variable assignments that improve fitness. 

\begin{table}[!ht]
\caption{LON Statistics for NKQ Landscapes with $n=$\numprint{10000} variables. Average (avg) and standard deviation (std) of the ten instances for each value of $K$ are shown.}
\label{tab:lon-stats}
\centering
\begin{tabular}{@{}rrrrrcrrr@{}}
\toprule
\multicolumn{2}{c}{} & \multicolumn{3}{c}{Crossover Edges} &\multicolumn{1}{c}{}  & \multicolumn{3}{c}{Perturbation Edges} \\
\cline{3-5}
\cline{7-9}
  & \multicolumn{1}{c}{Nodes} & \multicolumn{1}{c}{improv.} & \multicolumn{1}{c}{equal}  & \multicolumn{1}{c}{worse.} & & \multicolumn{1}{c}{improv.} & \multicolumn{1}{c}{equal}  & \multicolumn{1}{c}{worse.}  \\
\midrule
$K=2$ & & &   \\
\scriptsize avg & $\numprint{24361.7}$ & $\numprint{20490.7}$ & $\numprint{3676.9}$ & $0.0$ & & $10.8$  & $0.0$ & $\numprint{12263.4}$ \\ 
\scriptsize std  & $\numprint{2387.7}$ & $\numprint{1970.6}$ & $\numprint{1422.5}$ & $0.0$ & & $0.6$  & $0.0$ & $\numprint{1194.9}$ \\ 
$K=5$ & & &   \\ 
\scriptsize avg & $\numprint{23371.8}$ & $\numprint{22697.6}$ & $123.0$ & $0.0$ & & $53.2$ & $0.0$ & $\numprint{11898.8}$\\   
\scriptsize std  & $\numprint{3137.5}$ & $\numprint{3033.4}$ & $15.0$ & $0.0$ & & $5.4$ & $0.0$ & $\numprint{1612.8}$\\   
\bottomrule					
\end{tabular}
\end{table}

%
%
%

\begin{figure}[!ht]
\centering
\subfloat[][Start neigborhood, $K=2$]{
	\includegraphics[width=0.45\textwidth]{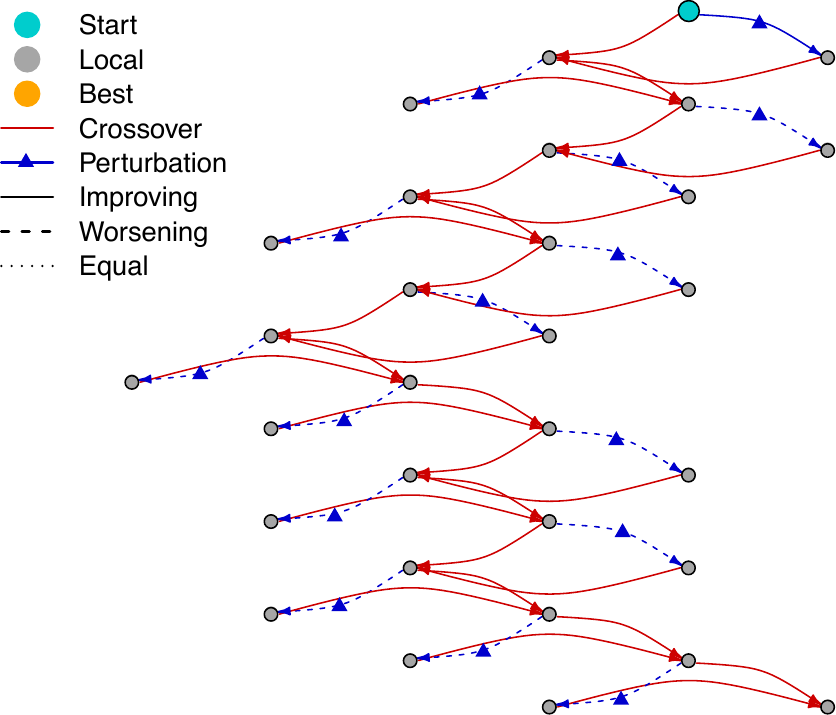}
}
\qquad
\subfloat[][Start neigborhood, $K=5$]{
	\includegraphics[width=0.45\textwidth]{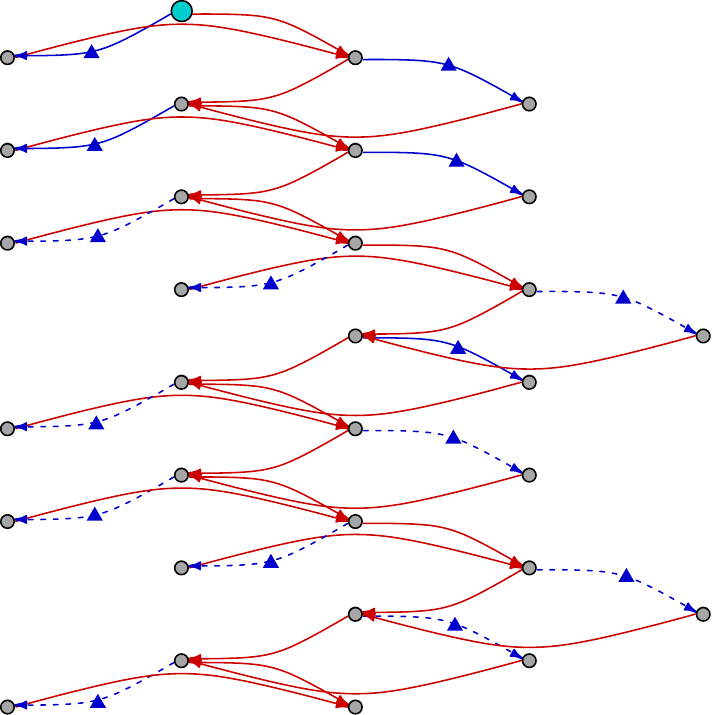}
}

\subfloat[][End neigborhood, $K=2$]{
	\includegraphics[width=0.45\textwidth]{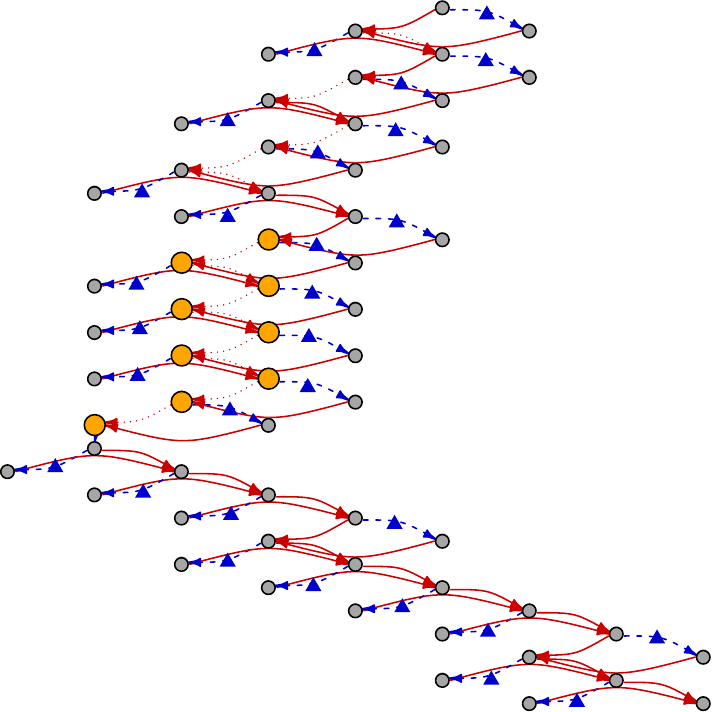}
}
\qquad
\subfloat[][End neigborhood, $K=5$]{
	\includegraphics[width=0.45\textwidth]{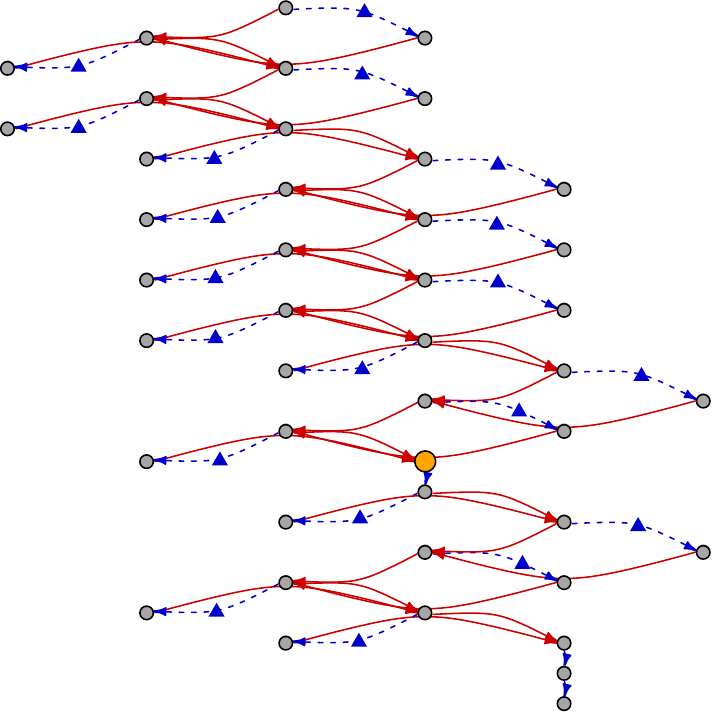}
}
\caption{LONs for two NKQ Landscapes instances with $n=\numprint{10000}$ and $K = 2, 5$ as indicated in the captions. Yellow nodes highlight the best local optima found, while cyan nodes the start local optimum. Red edges correspond to crossover and blue edges with triangular marker to perturbation. Improving edges are solid, worsening edges are dashed and equal edges are dotted.}
\label{fig:lons}
\end{figure}

In order to have a closer look at the search trajectories, Figure \ref{fig:lons} visualizes the LON of a  run obtaining the best solution for two selected instances with $n=\numprint{10000}$ and the two values of $K$. 
The plots use a tree-based graph layout where the root node is the local optimum at the beginning of the run, highlighted with cyan color. We color in yellow the best local optima found. 
Red edges are crossover edges, where each parent adds an edge pointing to the offspring. Blue edges with triangular marker represent perturbation edges. Solid, dashed and dotted line styles are used for improving, worsening and equal edges. Since the complete LONs are large (more than \numprint{2000} nodes in each run), we only show the neighbourhood around the starting and best node. 
The search process is very similar in structure for the two instances and reflects the working of the algorithm. However, we also observe some remarkable differences for different values of $K$. We observe more improving perturbation edges when $K=5$ than when $K=2$. In fact, only one improving perturbation edge appears for $K=2$, what means that it is difficult to improve a solution provided by DPX in the near neighborhood reached by the perturbation operator followed by local search. We also should highlight that improving perturbation edges appear only at the beginning of the search in both cases ($K=2$ and $K=5$). After a few steps, only DPX provides improvements over the best found solution. We also see that many different best local optima are found in $K=2$ joined by equal crossover edges, while only one appears in the instance with $K=5$. DPX shows here an ability to jump from one best local optimum to another. No equal crossover edges appear near the end node when $K=5$: once the best local optimum is found, a worsening perturbation edge makes the search to escape from it.

\section{Conclusions}
\label{sec:conclusions}

This paper proposes a new gray box crossover operator, dynastic potential crossover (DPX), with the ability to obtain a best offspring out of the full dynastic potential if the density of interactions among the variables is low, where it can behave like an optimal recombination operator. We have provided theoretical results proving that DPX, when recombining two parents, generates an offspring no worse than partition crossover (PX) and usually no worse than articulation points partition crossover (APX). We also performed a thorough experimental study to compare DPX with four crossover operators, including PX and APX, and using two different algorithms: DRILS and EA. We used NKQ Landscapes and MAX-SAT instances in the comparison.
We conclude that DPX offers a great exploration ability, providing offspring that are much better than their parents, compared to the other crossover operators. However, this ability has a cost in runtime and memory consumption, the main drawback of the operator. Included in DRILS, DPX outperforms all the other crossover operators in NKQ Landscapes and MAX-SAT instances. In the case of EA, it also outperforms the other crossover operators except in the case of NKQ Landscapes with $K=5$, where PX is the best performing operator. Thus, we suggest that a combined use of PX and DPX in EA could be optimal in high epistasis functions.

An interesting future line of research is to analyze the shape of the connected components of the recombination graph to design pre-computed clique trees that could speed up the operator. The code of DPX can also be much optimized when we focus on particular problems, like MAX-SAT. We used in this work a general implementation that can be optimized for particular problems. A problem-specialized version of DPX could be combined with state-of-the-art incomplete MAX-SAT solvers to improve their performance.

\section*{Acknowledgements}

This research is partially funded by the Universidad de M\'alaga, Consejer\'ia de Econom\'ia y Conocimiento de la Junta de Andaluc\'ia and FEDER under grant number UMA18-FEDERJA-003 (PRECOG); under grant PID 2020-116727RB-I00 (HUmove) funded by MCIN/AEI/ 10.13039/501100011033; and TAILOR ICT-48 Network (No 952215) funded by EU Horizon 2020 research and innovation programme.
The work is also partially supported in Brazil by S\~ao Paulo Research Foundation (FAPESP), under grants 2021/09720-2 and \mbox{2019/07665-4}, and National Council for Scientific and Technological Development (CNPq), under grant \mbox{305755/2018-8}.

Special thanks to J.A. Lozano, Sebastian Herrmann and Hansang Yun for providing pointers to algorithms for the clique tree construction, and to the anonymous reviewers for their thorough review and constructive suggestions.


\begin{thebibliography}{}

\bibitem[Bodlaender, 2005]{Bodlaender2005}
Bodlaender, H.~L. (2005).
\newblock Discovering treewidth.
\newblock In Vojt{\'a}{\v{s}}, P., Bielikov{\'a}, M., Charron-Bost, B., and
  S{\'y}kora, O., editors, {\em Proceedings of SOFSEM}, pages 1--16. Springer.

\bibitem[Cai and Lei, 2020]{CAI2020103354}
Cai, S. and Lei, Z. (2020).
\newblock Old techniques in new ways: Clause weighting, unit propagation and
  hybridization for maximum satisfiability.
\newblock {\em Artificial Intelligence}, 287:103354.

\bibitem[Chen and Whitley, 2017]{ChenWhitley2017}
Chen, W. and Whitley, D. (2017).
\newblock Decomposing {SAT} instances with pseudo backbones.
\newblock In Hu, B. and L{\'o}pez-Ib{\'a}{\~{n}}ez, M., editors, {\em
  Proceedings of EvoCOP}, volume 10197 of {\em LNCS}, pages 75--90. Springer.

\bibitem[Chen et~al., 2018]{Chen2018gecco}
Chen, W., Whitley, D., Tin\'{o}s, R., and Chicano, F. (2018).
\newblock Tunneling between plateaus: Improving on a state-of-the-art {MAXSAT}
  solver using partition crossover.
\newblock In Aguirre, H. and Takadama, K., editors, {\em Proceedings of GECCO},
  pages 921--928. ACM.

\bibitem[Chicano et~al., 2018]{Chicano2018gecco}
Chicano, F., Ochoa, G., Whitley, D., and Tin\'{o}s, R. (2018).
\newblock Enhancing partition crossover with articulation points analysis.
\newblock In Aguirre, H. and Takadama, K., editors, {\em Proceedings of GECCO},
  pages 269--276. ACM.

\bibitem[Chicano et~al., 2019]{Chicano2019EvoCOP}
Chicano, F., Ochoa, G., Whitley, D., and Tin{\'o}s, R. (2019).
\newblock Quasi-optimal recombination operator.
\newblock In Liefooghe, A. and Paquete, L., editors, {\em Proceedings of
  EvoCOP}, volume 11452 of {\em LNCS}, pages 131--146. Springer.

\bibitem[Chicano et~al., 2017]{ChicanoWOT17}
Chicano, F., Whitley, D., Ochoa, G., and Tin{\'{o}}s, R. (2017).
\newblock Optimizing one million variable {NK} landscapes by hybridizing
  deterministic recombination and local search.
\newblock In Ochoa, G. and Bosman, P., editors, {\em Proceedings of GECCO},
  pages 753--760. ACM.

\bibitem[Eremeev and Kovalenko, 2013]{Eremeev:Kovalenko2013}
Eremeev, A.~V. and Kovalenko, J.~V. (2013).
\newblock Optimal recombination in genetic algorithms.
\newblock {\em CoRR}, abs/1307.5519.

\bibitem[Galinier et~al., 1995]{Galinier1995}
Galinier, P., Habib, M., and Paul, C. (1995).
\newblock Chordal graphs and their clique graphs.
\newblock In Nagl, M., editor, {\em Graph-Theoretic Concepts in Computer
  Science}, pages 358--371. Springer.

\bibitem[Hammer et~al., 1963]{Hammer1963}
Hammer, P.~L., Rosenberg, I., and Rudeanu, S. (1963).
\newblock On the determination of the minima of pseudo-boolean functions.
\newblock {\em Studii \c{s}i Cercet\u{a}ri de Matematic\u{a}}, 14:359--364.

\bibitem[Hauschild and Pelikan, 2010]{Hauschild2010}
Hauschild, M.~W. and Pelikan, M. (2010).
\newblock Network crossover performance on {NK} landscapes and deceptive
  problems.
\newblock In Pelikan, M. and Branke, J., editors, {\em Proceedings of GECCO},
  pages 713--720. ACM.

\bibitem[Hordijk and Stadler, 1998]{Hordijk:1998ej}
Hordijk, W. and Stadler, P.~F. (1998).
\newblock {Amplitude Spectra of Fitness Landscapes}.
\newblock {\em Advances in Complex Systems}, 1:39--66.

\bibitem[{L\'opez-Ib\'a\~nez} et~al., 2016]{Lopez-Ibanez2016}
{L\'opez-Ib\'a\~nez}, M., Dubois-Lacoste, J., P\'erez-C\'aceres, L., Birattari,
  M., and St\"utzle, T. (2016).
\newblock The irace package: Iterated racing for automatic algorithm
  configuration.
\newblock {\em Operations Research Perspectives}, 3:43 -- 58.

\bibitem[Newman and Engelhardt, 1998]{Newman1998}
Newman, M. E.~J. and Engelhardt, R. (1998).
\newblock Effect of neutral selection on the evolution of molecular species.
\newblock {\em Proceedings of the Royal Society of London B}, 265:1333--1338.

\bibitem[Ochoa and Chicano, 2019]{Ochoa2019gecco}
Ochoa, G. and Chicano, F. (2019).
\newblock Local optima network analysis for {MAX-SAT}.
\newblock In L{\'{o}}pez{-}Ib{\'{a}}{\~{n}}ez, M., Auger, A., and
  St{\"{u}}tzle, T., editors, {\em Proceedings of GECCO Companion}, pages
  1430--1437. {ACM}.

\bibitem[Ochoa et~al., 2015]{Ochoa2015gecco}
Ochoa, G., Chicano, F., Tin\'{o}s, R., and Whitley, D. (2015).
\newblock Tunnelling crossover networks.
\newblock In Silva, S. and Esparcia-Alc\'azar, A.~I., editors, {\em Proceedings
  of GECCO}, pages 449--456. ACM.

\bibitem[Ochoa et~al., 2008]{lon-gecco08}
Ochoa, G., Tomassini, M., Verel, S., and Darabos, C. (2008).
\newblock A study of {NK} landscapes' basins and local optima networks.
\newblock In Ryan, C. and Keijzer, M., editors, {\em Proceedings of GECCO},
  pages 555--562. ACM.

\bibitem[Ochoa and Veerapen, 2018]{OchoaV-joh18}
Ochoa, G. and Veerapen, N. (2018).
\newblock Mapping the global structure of {TSP} fitness landscapes.
\newblock {\em Journal of Heuristics}, 24(3):265--294.

\bibitem[Radcliffe, 1994]{Radcliffe1994}
Radcliffe, N.~J. (1994).
\newblock The algebra of genetic algorithms.
\newblock {\em Annals of Mathematics and Artificial Intelligence},
  10(4):339--384.

\bibitem[Rana et~al., 1998]{Rana1998}
Rana, S., Heckendorn, R.~B., and Whitley, D. (1998).
\newblock A tractable walsh analysis of {SAT} and its implications for genetic
  algorithms.
\newblock In Mostow, J. and Rich, C., editors, {\em Proceedings of AAAI}, pages
  392--397.

\bibitem[Tarjan, 1975]{Tarjan1975}
Tarjan, R.~E. (1975).
\newblock Efficiency of a good but not linear set union algorithm.
\newblock {\em Journal of the ACM}, 22(2):215--225.

\bibitem[Tarjan and Yannakakis, 1984]{Tarjan:Yannakis1984}
Tarjan, R.~E. and Yannakakis, M. (1984).
\newblock {Simple linear-time algorithms to test chordality of graphs, test
  acyclicity of hypergraphs, and selectively reduce acyclic hypergraphs}.
\newblock {\em SIAM Journal on Computing}, 13(3):566--579.

\bibitem[Terras, 1999]{Terras1999}
Terras, A. (1999).
\newblock {\em Fourier Analysis on Finite Groups and Applications}.
\newblock Cambridge University Press.

\bibitem[Tin\'os et~al., 2015]{TinosWhitley2015}
Tin\'os, R., Whitley, D., and Chicano, F. (2015).
\newblock Partition crossover for pseudo-boolean optimization.
\newblock In He, J., Jansen, T., Ochoa, G., and Zarges, C., editors, {\em
  Proceedings of FOGA}, pages 137--149. ACM.

\bibitem[Tin\'os et~al., 2018]{Tinos2018tevc}
Tin\'os, R., Zhao, L., Chicano, F., and Whitley, D. (2018).
\newblock {NK} hybrid genetic algorithm for clustering.
\newblock {\em IEEE Transactions on Evolutionary Computation}, 22(5):748--761.

\bibitem[Whitley et~al., 2016]{Whitley2016ecj}
Whitley, D., Chicano, F., and Goldman, B.~W. (2016).
\newblock Gray box optimization for {Mk} landscapes ({NK} landscapes and
  {MAX-kSAT}).
\newblock {\em Evolutionary Computation}, 24:491 -- 519.

\bibitem[Wright et~al., 2000]{Wright2000}
Wright, A., Thompson, R., and Zhang, J. (2000).
\newblock The computational complexity of {NK} fitness functions.
\newblock {\em IEEE Transactions on Evolutionary Computation}, 4(4):373--379.

\end{thebibliography}

\end{document}